\documentclass{article}

\usepackage[final]{neurips_2022} 
\usepackage{titletoc}
\newcommand\DoToC{%
  \startcontents
  \printcontents{}{1}{\hrulefill\vskip0pt}
  \vskip0pt \noindent\hrulefill
  }
\usepackage[utf8]{inputenc} 
\usepackage[T1]{fontenc}    
\usepackage{courier}
\usepackage{hyperref}       
\usepackage{url}            

\usepackage{booktabs}       
\usepackage{amsfonts}       
\usepackage{amsmath} 
\usepackage{nicefrac}       
\usepackage{wrapfig}
\usepackage{multirow}

\usepackage{tikz}
\usetikzlibrary{fit,matrix,positioning, 
decorations.pathreplacing,calc,decorations,arrows,shadows,patterns}
\usetikzlibrary{arrows,bayesnet,fit,positioning,shapes,matrix}
\newdimen\arrowsize
\pgfarrowsdeclare{arcsq}{arcsq}
{
	\arrowsize=0.2pt
	\advance\arrowsize by .5\pgflinewidth
	\pgfarrowsleftextend{-4\arrowsize-.5\pgflinewidth}
	\pgfarrowsrightextend{.5\pgflinewidth}
}
{
	\arrowsize=0.8pt
	\advance\arrowsize by .5\pgflinewidth
	\pgfsetdash{}{0pt} 
	\pgfsetroundjoin   
	\pgfsetroundcap    
	\pgfpathmoveto{\pgfpoint{0\arrowsize}{0\arrowsize}}
	\pgfpatharc{-90}{-140}{4\arrowsize}
	\pgfusepathqstroke
	\pgfpathmoveto{\pgfpointorigin}
	\pgfpatharc{90}{140}{4\arrowsize}
	\pgfusepathqstroke
}

\usepackage{amsmath}
\usepackage{amssymb}
\usepackage{mathtools}
\usepackage{amsthm}
\usepackage{textcomp}

\newcommand{\x}{\mathbf{x}}
\newcommand{\z}{\mathbf{z}}

\newcommand{\ind}{\perp\!\!\!\!\perp} 
\def\Eqref#1{Eq.~\ref{#1}}

\theoremstyle{plain}
\newtheorem{definition}{Definition}
\newtheorem{corollary}{Corollary}

\newtheorem{proposition}{Proposition}

\newtheorem{theorem}{Theorem}

\usepackage{xcolor,colortbl}
\definecolor{ourmethod}{gray}{0.93}
\definecolor{myredcolor}{RGB}{215,48,39}
\definecolor{mygreencolor}{RGB}{26,152,80}
\usepackage{pifont}
\newcommand{\cmark}{\textcolor{black}{\ding{51}}}
\newcommand{\xmark}{\textcolor{black}{\ding{55}}}

\usepackage{enumitem}

\usepackage{caption}
\usetikzlibrary{backgrounds}

\usepackage{microtype}      

\newcommand{\ourtitle}{Temporally Disentangled Representation Learning}
\newcommand{\ourmeos}{\textbf{\texttt{TDRL}} }

\title{\ourtitle}

\author{%
  Weiran Yao\\
  CMU\\
  \texttt{weiran@cmu.edu} \\
  \And
  Guangyi Chen \\
  CMU \& MBZUAI \\
  \texttt{guangyichen1994@gmail.com} \\
  \And
  Kun Zhang \\
  CMU \& MBZUAI \\
  \texttt{kunz1@cmu.edu} \\
}

\begin{document}

\maketitle

\begin{abstract}

Recently in the field of unsupervised representation learning, strong identifiability results for disentanglement of causally-related latent variables have been established by exploiting certain side information, such as class labels, in addition to independence. However, most existing work is constrained by functional form assumptions such as independent sources or further with linear transitions, and distribution assumptions such as stationary, exponential family distribution. It is unknown whether the underlying latent variables and their causal relations are identifiable if they have arbitrary, nonparametric causal influences in between.  In this work, we establish the identifiability theories of nonparametric latent causal processes from their nonlinear mixtures under fixed temporal causal influences and analyze how distribution changes can further benefit the disentanglement. We propose \textbf{\texttt{TDRL}}, a principled framework to recover time-delayed latent causal variables and identify their relations from measured sequential data under stationary environments and under different distribution shifts. Specifically, the framework can factorize unknown distribution shifts into transition distribution changes under fixed and time-varying latent causal relations, and under observation changes in observation. Through experiments, we show that time-delayed latent causal influences are reliably identified and that our approach considerably outperforms existing baselines that do not correctly exploit this modular representation of changes. Our code is available at: \url{https://github.com/weirayao/tdrl}.

\end{abstract}


\section{Introduction}
Causal reasoning for time-series data is a fundamental task in numerous fields \cite{berzuini2012causality,ghysels2016testing,friston2009causal}. Most existing work focuses on estimating the temporal causal relations among observed variables. However, in many real-world scenarios, the observed signals (e.g., image pixels in videos) do not have direct causal edges, but are generated by latent temporal processes or confounders that are causally related. Inspired by these scenarios, this work aims to uncover causally-related latent processes and their relations from  observed temporal variables. Estimating latent causal structure from observations, which we assume are unknown (but invertible) nonlinear mixtures of the latent processes, is very challenging. It has been found in \cite{locatello2019challenging,hyvarinen1999nonlinear} that without exploiting an appropriate class of assumptions in estimation, the latent variables are not identifiable in the most general case. As a result, one cannot make causal claims on the recovered relations in the latent space.

Recently, in the field of unsupervised representation learning, strong identifiability results of the latent variables have been established \cite{hyvarinen2016unsupervised,hyvarinen2017nonlinear,hyvarinen2019nonlinear,khemakhem2020variational,sorrenson2020disentanglement} by using certain side information in nonlinear Independent Component Analysis (ICA), such as class labels, in addition to independence. For time-series data, history information is widely used as the side information for the identifiability of latent processes. To establish identifiability, the existing approaches enforce different sets of functional and distributional form assumptions as constraints in estimation; for example, \textbf{(1)} PCL \cite{hyvarinen2017nonlinear}, GCL \cite{hyvarinen2019nonlinear}, HM-NLICA \cite{halva2020hidden} and SlowVAE \cite{klindt2020towards} assume mutually-independent sources in the data generating process. However, this assumption may severely distort the identifiability if the latent variables have time-delayed causal relations in between (i.e., causally-related process); \textbf{(2)} SlowVAE \cite{klindt2020towards} and SNICA \cite{halva2021disentangling} assume linear relations, which may distort the identifiability results if the underlying transitions are nonlinear, and \textbf{(3)} SlowVAE \cite{klindt2020towards} assumes that the process noise is drawn from Laplacian distribution; i-VAE \cite{khemakhem2020variational} assumes that the conditional transition distribution is part of the exponential family. However, in real-world scenarios, one cannot choose a proper set of functional and distributional form assumptions without knowing in advance the parametric forms of the latent temporal processes. Our first step is hence to understand under what conditions the  latent causally processes are identifiable if they have \underline{nonparametric transitions} in between. With the proposed condition, our approach allows recovery of latent temporal \underline{causally-related processes} in \underline{stationary environments} without knowing their parametric forms in advance. 
\begin{wrapfigure}{r}{9cm}
\vspace{-0.25cm}
\centering
\includegraphics[width=\linewidth]{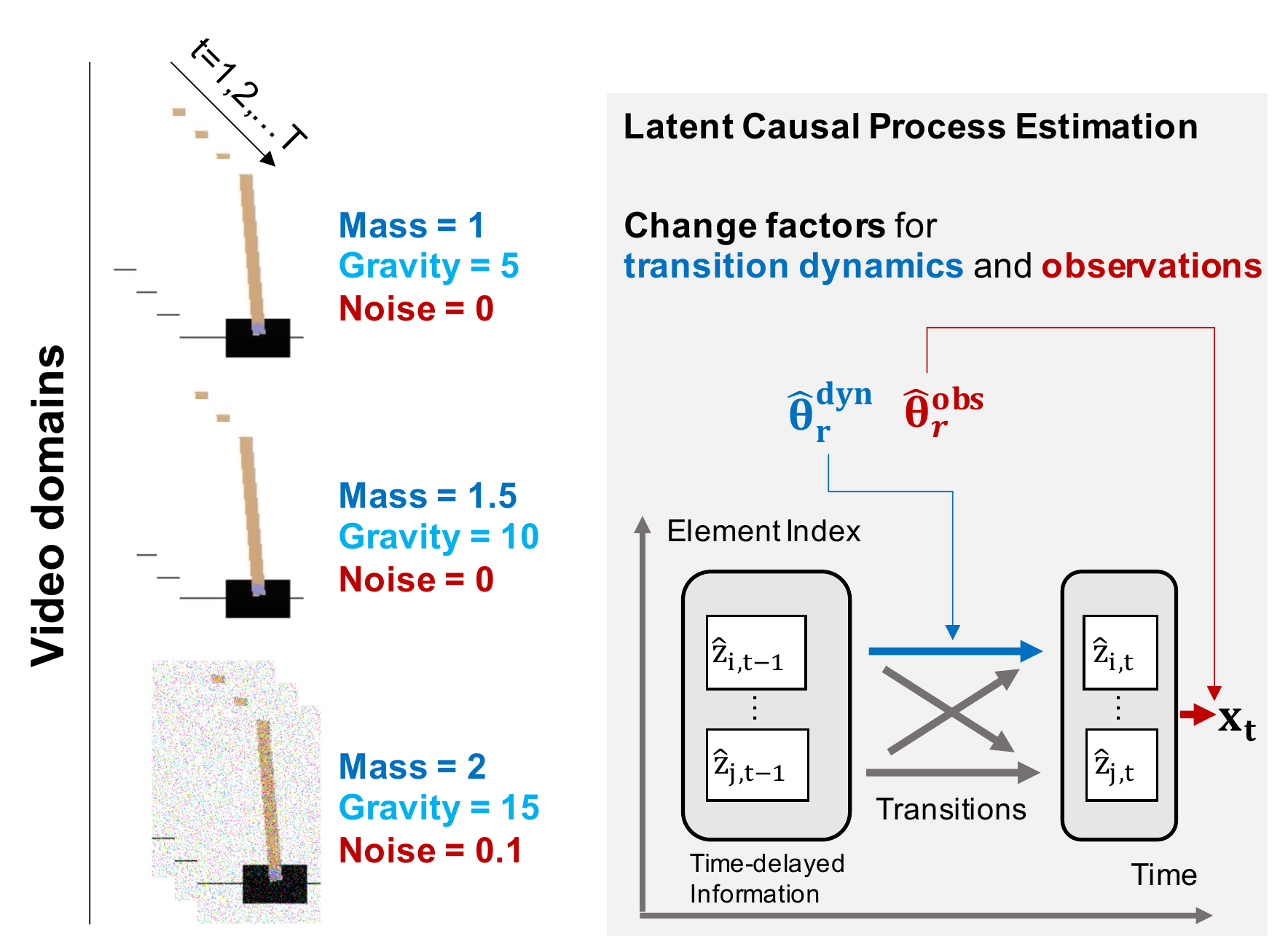}
\caption{\textbf{\texttt{TDRL}}: \underline{T}emporally \underline{D}isentangled \underline{R}epresentation \underline{L}earning. We exploit fixed causal dynamics and distribution changes from changing causal influences and global observation changes to identify the underlying causal processes.  $\hat{z}_{i,t}$ is the estimated latent process. $\boldsymbol{\hat{\theta}}_r^{\text{dyn}}$ is the change factor for \textcolor{blue}{\bf transition dynamics}, i.e., representing mass and gravity in this example. $\boldsymbol{\hat{\theta}}_r^{\text{obs}}$ is the change factor for \textcolor{red}{\bf observation}, i.e.,  noise scale.}
\label{fig:problem}
\vspace{-0.75cm}
\end{wrapfigure}
On the other hand, nonstationarity has greatly improved the identifiability results for learning the latent causal structure \cite{yao2021learning,bengio2019meta, ke2019learning}. For instance, LEAP \cite{yao2021learning} established the identifiability of latent temporal processes, but in limited nonstationary cases, under the condition that the distribution of the noise terms of the latent processes varies across all segments. Our second step is to analyze how distribution shifts benefit our stationary condition and to extend our condition to a general nonstationary case. Accordingly, our approach enables the recovery of latent temporal causal processes in a general nonstationary environment with \underline{time-varying relations} such as changes in the influencing strength or switching some edges off \cite{li2020causal} over time or domains.
\begin{table}[t]
\vspace{-0.8cm}
\centering
\begin{minipage}[b]{\linewidth}
\centering
\caption{Attributes of nonlinear ICA theories for time-series. A check denotes that a method has an attribute or can be applied to a setting, whereas a cross denotes the opposite. $^\dagger$ indicates our approach.\label{tab:prior}}
\resizebox{\textwidth}{!}{
\begin{tabular}{lccccc}
\toprule
{\bf Theory} &
  \begin{tabular}[c]{@{}c@{}}\bf Time-varying  \\ \bf Relation\end{tabular} &
  \begin{tabular}[c]{@{}c@{}}\bf Causally-related \\ \bf Process\end{tabular} &
  \begin{tabular}[c]{@{}c@{}}\bf Partitioned \\ \bf Subspace\end{tabular} &
  \begin{tabular}[c]{@{}c@{}}\bf Nonparametric\\ \bf Transition\end{tabular}  &
  \begin{tabular}[c]{@{}c@{}}\bf Applicable to \\ \bf Stationary Environment \end{tabular}  
  \\ \hline
{PCL} & \xmark & \xmark & \xmark & \cmark &   \cmark \\
{GCL}  & \cmark & \xmark & \xmark & \cmark &   \cmark  \\
{HM-NLICA} & \xmark & \xmark & \xmark & \xmark &  \xmark \\
{SlowVAE}  & \xmark & \xmark & \xmark & \xmark &  \cmark \\
{SNICA} &\cmark & \cmark & \xmark &\xmark & \xmark \\
{i-VAE} &\cmark & \xmark & \xmark &\xmark & \xmark \\
{LEAP} &\xmark & \cmark & \xmark & \cmark & \xmark \\
\rowcolor{ourmethod} {\ourmeos$^\dagger$} & \cmark & \cmark & \cmark & \cmark & \cmark \\
\bottomrule
\end{tabular}
}
\end{minipage}
\vspace{-0.5cm}
\end{table}

Given the identifiability results, we propose a learning framework, called \ourmeos, to recover nonparametric time-delayed latent causal variables and identify their relations from measured temporal data under stationary environments and under nonstationary environments in which it is unknown in advance how the joint distribution changes across domains (we define it as ``unknown distribution shifts''). For instance,  Fig.~\ref{fig:problem} shows an example of multiple video domains of a physical system under different mass, gravity, and environment rendering settings\footnote{The variables and functions with ``hat''  are estimated by the model; the ones without ``hat''  are ground truth.}. With \ourmeos, the differences across segments are characterized by the learned change factors $\boldsymbol{\hat{\theta}}_r^{\text{dyn}}$ of domain $r$ (note that domain index is given to the model) that encode changes in transition dynamics, and changes in observation or styles modeled by $\boldsymbol{\hat{\theta}}_r^{\text{obs}}$ (we use ``causal dynamics'' and ``latent causal relations/influences'' interchangeably). We then present a generalized time-series data generative model that takes these change factors as arguments for modeling the distribution changes. Specifically, we factorize unknown distribution shifts into transition distribution changes in stationary processes, time-varying latent causal relations, and global changes in observation by constructing partitioned latent subspaces, and propose provable conditions under which nonparametric latent causal processes can be identified from their nonlinear invertible mixtures. We demonstrate through a number of real-world datasets, including video and motion capture data, that time-delayed latent causal influences are reliably identified from observed variables under stationary environments and unknown distribution shifts. Through experiments, we show that our approach considerably outperforms existing baselines that do not correctly leverage this modular representation of changes. 

\section{Related Work}
\paragraph{Causal Discovery from Time Series} 
Inferring the causal structure from time-series data is critical to many fields including machine learning~\cite{berzuini2012causality},  econometrics~\cite{ghysels2016testing}, and neuroscience~\cite{friston2009causal}. Most existing work focuses on estimating the temporal causal relations between observed variables. For this task, constraint-based methods~\cite{entner2010causal} apply the conditional independence tests to recover the causal structures, while score-based methods~\cite{murphy2002dynamic,pamfil2020dynotears}
define score functions to guide a search process. Furthermore, 
~\cite{malinsky2018causal,malinsky2019learning} propose to fuse both conditional independence tests and score-based methods. The Granger causality~\cite{granger1969investigating} and its nonlinear variations~\cite{tank2018neural, lowe2020amortized} are also widely used.

\paragraph{Nonlinear ICA for Time Series}
Temporal structure and nonstationarities were recently used to achieve identifiability in nonlinear ICA. Time-contrastive learning (TCL \cite{hyvarinen2016unsupervised}) used the independent sources assumption and leveraged sufficient variability in variance terms of different data segments. Permutation-based contrastive (PCL \cite{hyvarinen2017nonlinear}) proposed a learning framework which discriminates between true independent sources and permuted ones, and identifiable under the uniformly dependent assumption. HM-NLICA \cite{halva2020hidden} combined nonlinear ICA with a Hidden Markov Model (HMM) to automatically model nonstationarity without  manual data segmentation. i-VAE \cite{khemakhem2020variational} introduced VAEs to approximate the true joint distribution over observed and auxiliary nonstationary regimes. Their work assumes that the conditional distribution is within exponential families to achieve the identifiability of the latent space. The most recent literature on nonlinear ICA for time-series includes LEAP \cite{yao2021learning} and (i-)CITRIS \cite{lippe2022citris,lippe2022icitris}. LEAP proposed a nonparametric condition leveraging the nonstationary noise terms. However, all latent processes are changed across contexts and the distribution changes need to be modeled by nonstationary noise and it does not exploit the stationary nonparametric components for identifiability.  Alternatively, CITRIS proposed to use intervention target information for identification of scalar and multidimensional latent causal factors. This approach does not suffer from functional or distributional form constraints, but needs access to active intervention.
\section{Problem Formulation} 
\subsection{Time Series Generative Model}
\paragraph{Stationary Model} As a \textbf{fundamental} case, we first present a regular, stationary time-series generative process  where the observations $\mathbf{x}_t$ comes from a nonlinear (but invertible) mixing function $\mathbf{g}$ that maps the time-delayed causally-related latent variables $\mathbf{z}_t$ to $\mathbf{x}_t$. The latent variables or processes $\mathbf{z}_t$ have stationary, nonparametric time-delayed causal relations. Let $\tau$ be the time lag:
\begin{equation*}
   \underbrace{ \mathbf{x}_t = \mathbf{g}(\mathbf{z}_t) }_{\text{Nonlinear mixing}}, \quad \underbrace{z_{it} = f_i\left(\{z_{j, t-\tau} \vert z_{j, t-\tau} \in \mathbf{Pa}(z_{it}) \}, \epsilon_{it}  \right)}_{\text{Stationary nonparametric transition}} \; with \underbrace{\epsilon_{it} \sim p_{\epsilon_i}}_{\text{Stationary noise}}.
\end{equation*}
Note that with nonparametric causal transitions, the noise term $\epsilon_{it} \sim p_{\epsilon_i}$ (where $p_{\epsilon_i}$ denotes the distribution of $\epsilon_{it}$) and the time-delayed parents $\mathbf{Pa}(z_{it})$ of $z_{it}$ (i.e., the set of latent factors that directly cause $z_{it}$)  are interacted and transformed in an arbitrarily nonlinear way to generate $z_{it}$. Under stationarity assumptions, the mixing function $\mathbf{g}$, the transition functions $f_i$ and the noise distributions $p_{\epsilon_i}$ are invariant. Finally, we assume that the noise terms are mutually-independent (i.e., spatially and temporally independent), which implies that instantaneous causal influence between latent causal processes is not allowed by the formulation. The stationary time-series model in the fundamental case is used to establish the identifiability results under fixed causal dynamics in Section~\ref{sec:fixed}. 

\paragraph{Nonstationary Model} We further consider two violations of the stationarity assumptions in the fundamental case, which lead to two nonstationary time series models. Let $\mathbf{u}$ denote the  domain or regime index. Suppose there exist $m$ regimes of data, i.e., $u_r$ with $r=1,2,...,m$, with unknown distribution shifts. In practice, the changing parameters of the joint distribution across domains often lie in a low-dimensional manifold \cite{stojanov2019data}. Moreover, if the distribution is causally factorized, the distributions are often changed in a minimal and sparse way \cite{ghassami2018multi}. Based on these assumptions, we introduce the low-dimensional minimal change factor $(\boldsymbol{\theta}_r^{\text{dyn}}, \boldsymbol{\theta}_r^{\text{obs}})$, which was proposed in \cite{huang2021adarl}, to respectively capture distribution shifts in transition functions and observation. The vector $\boldsymbol{\theta}_r = (\boldsymbol{\theta}_r^{\text{dyn}}, \boldsymbol{\theta}_r^{\text{obs}})$ has a constant value in each domain but varies across domains. The formulation of the nonstationary time-series model is in line with \cite{huang2021adarl}. The nonstationary model is used to establish the identifiability results under nonstationary cases in Section~\ref{sec:change}, where we show that the violation of stationarity in both ways can even further improve the identifiability results. We first present the two nonstationary cases. \textbf{(1) Changing Causal Dynamics}. The causal influences between the latent temporal processes are changed across domains in this setting. We model it by adding the transition change factors $\boldsymbol{\theta}_r^{\text{dyn}}$ as inputs to the transition function: $z_{it} = f_i\left(\{z_{j, t-\tau} \vert z_{j, t-\tau} \in \mathbf{Pa}(z_{it}) \}, \boldsymbol{\theta}_r^{\text{dyn}}, \epsilon_{it}  \right)$. \textbf{(2) Global Observation Changes.} The global properties of the  time series (e.g., video styles) are changed across domains in this setting. Our model captures them using latent variables that represent global styles; these latent variables are generated by a bijection $f_i$ that transforms the noise terms $\epsilon_{i,t}$ into the latent with change factor $\boldsymbol{\theta}_r^{\text{obs}}$: $z_{i,t} = f_i\left(\boldsymbol{\theta}_r^{\text{obs}}, \epsilon_{i,t}  \right)$. Finally, we can deal with a more general nonstationary case by combining the three types of latent processes in the latent space in a modular way. \textbf{(3) Modular Distribution Shifts.}
\begin{wrapfigure}{r}{0.6\textwidth}
\begin{equation}
 \small
\vspace{-0.2cm}
\setlength{\abovedisplayskip}{4pt}
\setlength{\belowdisplayskip}{4pt}
\left \{
 \begin{array}{lll}
     z_{s, t}^{\text{fix}} &= f_s\left(\{z_{i, t-\tau} \vert z_{i, t-\tau} \in \mathbf{Pa}(z_{s,t}^{\text{fix}}) \}, \epsilon_{s,t}  \right), \\
     z_{c, t}^{\text{chg}} &= f_c\left(\{z_{i, t-\tau} \vert z_{i, t-\tau} \in \mathbf{Pa}(z_{c,t}^{\text{chg}}) \}, \boldsymbol{\theta}_r^{\text{dyn}}, \epsilon_{c,t}  \right), \\
     z_{o, t}^{\text{obs}} &= f_o\left(\boldsymbol{\theta}_r^{\text{obs}}, \epsilon_{o,t}  \right), \\
    \mathbf{x}_t &= \mathbf{g}(\mathbf{z}_t).
 \end{array}\right.
 \label{eq:model}
\end{equation} 
\end{wrapfigure}
The latent space has three blocks $\mathbf{z}_t = (\mathbf{z}_t^{\text{fix}}, \mathbf{z}_t^{\text{chg}}, \mathbf{z}_t^{\text{obs}})$ where $z_{s,t}^{\text{fix}}$ is the s\textsuperscript{th} component of the fixed dynamics parts, $z_{c,t}^{\text{chg}}$ is the c\textsuperscript{th} component of the changing dynamics parts, and $\mathbf{z}_{o,t}^{\text{obs}}$ is the o\textsuperscript{th} component of the observation changes. The functions $[f_s,f_c,f_o]$  capture fixed and changing transitions and observation changes for each dimension of $\mathbf{z}_t$ in Eq.~\ref{eq:model}.
\subsection{Identifiability of Latent Causal Processes and Time-Delayed Latent Causal Relations}

We define the identifiability of time-delayed latent causal processes in the representation function space in \textbf{Definition 1}. Furthermore, if the estimated latent processes can be identified at least up to permutation and component-wise invertible nonlinearities, the latent causal relations are also immediately identifiable because conditional independence relations fully characterize time-delayed causal relations in a time-delayed causally sufficient system, in which there are no latent causal confounders in the (latent) causal
processes. Note that invertible component-wise transformations on latent causal processes do not change their conditional independence
relations.

\begin{definition}[Identifiable Latent Causal Processes]
Formally let $\{\mathbf{x}_t\}_{t=1}^T$ be a sequence of observed variables generated by the true temporally causal latent processes specified by $(f_i, \boldsymbol{\theta}_r, p({\epsilon_i}), \mathbf{g})$ given in Eq.~\ref{eq:model}. A learned generative model $(\hat{f}_i, \hat{\boldsymbol{\theta}}_r, \hat{p}({\epsilon_i}), \hat{\mathbf{g}})$ is observationally equivalent to $(f_i, \boldsymbol{\theta}_r, p({\epsilon_i}), \mathbf{g})$ if the model distribution $p_{\hat{f}, \hat{\boldsymbol{\theta}}_r, \hat{p}_\epsilon, \hat{\mathbf{g}}}(\{\mathbf{x}_t\}_{t=1}^T)$ matches the data distribution $ p_{f, \boldsymbol{\theta}_r, p_\epsilon, \mathbf{g}}(\{\mathbf{x}_t\}_{t=1}^T)$ everywhere. We say latent causal processes are identifiable if observational equivalence can lead to identifiability of the latent variables up to permutation $\pi$ and component-wise invertible transformation $T$:
\begin{equation}\label{eq:iden}
\setlength{\abovedisplayskip}{4pt}
\setlength{\belowdisplayskip}{4pt}
\begin{aligned}
p_{\hat{f}_i, \hat{\boldsymbol{\theta}}_r, \hat{p}_{\epsilon_i}, \hat{\mathbf{g}}}(&\{\mathbf{x}_t\}_{t=1}^T) =  p_{f_i, \boldsymbol{\theta}_r, p_{\epsilon_i}, \mathbf{g}} (\{\mathbf{x}_t\}_{t=1}^T) 
\Rightarrow \hat{\mathbf{g}}(\mathbf{x}_t) = \mathbf{g} \circ \pi \circ T, \quad \forall \mathbf{x}_t \in \mathcal{X},
\end{aligned}
\end{equation}
where $\mathcal{X}$ is the observation space.
\end{definition} 

%

\section{Identifiability Theory}\label{sec:theory} 

We establish the identifiability theory of nonparametric time-delayed latent causal processes under three different types of distribution shifts. W.l.o.g., we consider the latent processes with maximum time lag $L=1$. The extentions to arbitrary time lags are discussed in Appendix~\ref{sec:mlag-ap}. Let $k$ be the element index of the latent space $\mathbf{z}_t$ and the latent size is $n$. In particular, \textbf{(1)} under fixed temporal causal influences, we leverage the distribution changes $p(z_{k,t} \vert \mathbf{z}_{t-1})$ for different values of $\mathbf{z}_{t-1}$; \textbf{(2)} when the underlying causal relations change over time, we exploit the changing causal influences on $p(z_{k,t} \vert \mathbf{z}_{t-1}, u_r)$ under different domain $u_r$, and \textbf{(3)} under global observation changes, the nonstationarity $p(z_{k,t} \vert u_r)$ under different values of $u_r$ is exploited. The proofs are provided in Appendix~\ref{ap:theory}. The comparisons between existing theories are in Appendix~\ref{ap:compar}.


\subsection{Identifiability under Fixed Temporal Causal Influence}\label{sec:fixed}
Let $\eta_{kt} \triangleq \log p(z_{k, t} | \mathbf{z}_{t-1})$. Assume that $\eta_{kt}$ is twice differentiable in $z_{k,t}$ and is differentiable in $z_{l,t-1}$, $l=1,2,...,n$. Note that the parents of $z_{k,t}$ may be only a subset of $\mathbf{z}_{t-1}$; if $z_{l,t-1}$ is not a parent of $z_{k,t}$, then $\frac{\partial \eta_{kt}}{\partial z_{l,t-1}} = 0$.
Below we provide a {\it sufficient condition} for the identifiability of $\mathbf{z}_{t}$, followed by a discussion of specific unidentifiable and identifiable cases to illustrate how general it is.
 
\begin{theorem}[Identifiablity under a Fixed Temporal Causal Model] \label{Th1} 
Suppose there exists invertible function $\hat{\mathbf{g}}$  that maps $\mathbf{x}_t$ to $\hat{\mathbf{z}}_t$, i.e., 
\begin{equation} \label{eq:invert}
    \hat{\mathbf{z}}_t = \hat{\mathbf{g}}(\mathbf{x}_t)
\end{equation}
such that the components of $\hat{\mathbf{z}}_t$ are mutually  independent conditional on $\hat{\mathbf{z}}_{t-1}$. 
Let 

\begin{equation}
\label{eq:v}
\setlength{\abovedisplayskip}{1pt}
\setlength{\belowdisplayskip}{1pt}
\resizebox{.485\hsize}{!}{
$\mathbf{v}_{k,t} \triangleq \Big(\frac{\partial^2 \eta_{kt}}{\partial z_{k,t} \partial z_{1,t-1}}, \frac{\partial^2 \eta_{kt}}{\partial z_{k,t} \partial z_{2,t-1}}, ..., \frac{\partial^2 \eta_{kt}}{\partial z_{k,t} \partial z_{n,t-1}} \Big)^\intercal$},
\resizebox{.485\hsize}{!}{
$\mathring{\mathbf{v}}_{k,t} \triangleq \Big(\frac{\partial^3 \eta_{kt}}{\partial z_{k,t}^2 \partial z_{1,t-1}}, \frac{\partial^3 \eta_{kt}}{\partial z_{k,t}^2 \partial z_{2,t-1}}, ..., \frac{\partial^3 \eta_{kt}}{\partial z_{k,t}^2 \partial z_{n,t-1}} \Big)^\intercal$.}
\end{equation}
If for each value of $\mathbf{z}_t$, $\mathbf{v}_{1,t}, \mathring{\mathbf{v}}_{1,t}, \mathbf{v}_{2,t}, \mathring{\mathbf{v}}_{2,t}, ..., \mathbf{v}_{n,t}, \mathring{\mathbf{v}}_{n,t}$, as 
$2n$ vector functions in $z_{1,t-1}$, $z_{2,t-1}$, ..., $z_{n,t-1}$, are linearly independent, then $\mathbf{z}_{t}$ must be an invertible, component-wise transformation of a permuted version of $\hat{\mathbf{z}}_t$.

\end{theorem}

The linear independence condition in Theorem \ref{Th1} is the core condition to guarantee the identifiability of $\mathbf{z}_t$ from the observed $\mathbf{x}_t$. To make this condition more intuitive, below we consider specific unidentifiable cases, in which there is no temporal dependence in $\mathbf{z}_t$ or the noise terms in $\mathbf{z}_t$ are additive Gaussian, and two identifiable cases, in which $\mathbf{z}_t$ has additive, heterogeneous noise or follows some linear, non-Gaussian temporal process. 

Let us start with two unidentifiable cases. In case $\texttt{N}1$, $\mathbf{t}_t$ is an independent and identically distributed (i.i.d.) process, i.e., there is no causal influence from any component of $\mathbf{z}_{t-1}$ to any $z_{k,t}$. In this case,  $\mathbf{v}_{k,t}$ and $\mathring{\mathbf{v}}_{k,t}$ (defined in Eq.~\ref{eq:v}) are always $\mathbf{0}$ for $k=1,2,...,n$, since $p(z_{k,t} \,|\, \mathbf{z}_{t-1})$ does not involve $\mathbf{z}_{t-1}$. So the linear independence condition is violated. In fact, this is the regular nonlinear ICA problem with i.i.d. data, and it is well-known that the underlying independent variables are not identifiable \cite{hyvarinen1999nonlinear}. In case $\texttt{N}_2$, all $z_{k,t}$ follow an additive noise model with Gaussian noise terms, i.e., 
\begin{equation} \label{eq:Gaussian_case}
    \mathbf{z}_t = \mathbf{q}(\mathbf{z}_{t-1}) + \mathbf{\epsilon}_t,
\end{equation}
where $\mathbf{q}$ is a transformation and the components of the Gaussian vector $\mathbf{\epsilon}_t$ are independent and also independent from $\mathbf{z}_{t-1}$. Then  $\frac{\partial^2 \eta_{kt}}{\partial z_{k,t}^2}$ is constant, and $\frac{\partial^3 \eta_{kt}}{\partial z_{k,t}^2 \partial z_{l,t-1}} \equiv 0$, violating the linear independence condition. In the following proposition we give some alternative solutions and verify the unidentifiability in this case.

\begin{proposition}[Unidentifiability under Gaussian Noise]

Suppose $\mathbf{x}_t = \mathbf{g}(\mathbf{z}_t)$ was generated by Eq.~\ref{eq:Gaussian_case}, where the components of $\mathbf{\epsilon}_t$ are mutually independent Gaussian and also independent from $\mathbf{z}_{t-1}$. Then any $\hat{\mathbf{z}}_t = \mathbf{D}_1 \mathbf{U} \mathbf{D}_2 \cdot {\mathbf{z}}_t$, where $\mathbf{D}_1$ is an arbitrary non-singular diagonal matrix, $\mathbf{U}$ is an arbitrary orthogonal matrix, and $\mathbf{D}_2$ is a diagonal matrix with $\mathbb{V}ar^{-1/2}(\epsilon_{k,t})$ as its $k^{\text{th}}$  diagonal entry, is a valid solution to satisfy the condition that the components of $\hat{\mathbf{z}}_t$ are mutually independent conditional on $\hat{\mathbf{z}}_{t-1}$.

\end{proposition}
Roughly speaking, for a randomly chosen conditional density function $p(z_{k,t}\,|\,\mathbf{z}_{t-1})$ in which $z_{k,t}$ is not independent from $\mathbf{z}_{t-1}$ (i.e., there is temporal dependence in the latent processes) and which does not follow an additive noise model with Gaussian noise, the chance for its specific second- and third-order partial derivatives to be linearly dependent is slim. Now let us consider two cases in which the latent temporally processes $\mathbf{z}_t$ are naturally identifiable. First, consider case $\texttt{Y}1$, where $z_{k,t}$ follows a heterogeneous noise process, in which the noise variance depends on its parents:
\begin{equation} \label{eq:heteo}
    z_{k,t} = q_k(\mathbf{z}_{t-1}) + \frac{1}{b_k(\mathbf{z}_{t-1})}\epsilon_{k,t}.
\end{equation}
Here we assume $\epsilon_{k,t}$ is standard Gaussian and $\epsilon_{1,t}, \epsilon_{2,t}, .., \epsilon_{n,t}$ are mutually independent and independent from $\mathbf{z}_{t-1}$. $\frac{1}{b_k}$, which depends on $\mathbf{z}_{t-1}$, is the standard deviation of the noise in $z_{k,t}$. (For conciseness, we drop the argument of $b_k$ and $q_k$ when there is no confusion.) Note that in this model, if $q_k$ is 0 for all $k=1,2,...,n$, it reduces to a multiplicative noise model. 
The identifiability result of $\mathbf{z}_t$ is established in the following corollary.

\begin{corollary}[Identifiability under Heterogeneous Noise]

Suppose $\mathbf{x}_t = \mathbf{g}(\mathbf{z}_t)$ was generated according to Eq.~\ref{eq:heteo}, and Eq.~\ref{eq:invert} holds true. If  $b_k\cdot \frac{\partial b_k}{\partial \mathbf{z}_{t-1}}$ and $b_k\cdot \frac{\partial b_k}{\partial \mathbf{z}_{t-1}} (z_{k,t} - q_{k}) - b_k^2\cdot \frac{\partial q_{k}}{\partial \mathbf{z}_{t-1}}$, with $k=1,2,...,n$, which are in total $2n$ function vectors in $\mathbf{z}_{t-1}$, 
are linearly independent, then $\mathbf{z}_{t}$ must be an invertible, component-wise transformation of a permuted version of $\hat{\mathbf{z}}_t$.

\end{corollary}

Let us then consider another special case, denoted by $\texttt{Y}2$, with a linear, non-Gaussian temporal model for $\mathbf{z}_t$: the latent processes follow Eq.~\ref{eq:Gaussian_case}, with $\mathbf{q}$ being a linear transformation and $\epsilon_{k,t}$ following a particular class of non-Gaussian distributions. 
The following corollary shows that $\mathbf{z}_t$ is identifiable as long as each $z_{k,t}$ receives causal influences from some components of $\mathbf{z}_{t-1}$.

\begin{corollary}[Identifiability under a Specific Linear, Non-Gaussian Model for Latent Processes]

Suppose $\mathbf{x}_t = \mathbf{g}(\mathbf{z}_t)$ was generated according to Eq.~\ref{eq:Gaussian_case}, in which $\mathbf{q}$ is a linear transformation and for each $z_{k,t}$, there exists at least one $k'$ such that $c_{k,k'} \triangleq \frac{\partial z_{k,t}}{\partial z_{k',t-1}} \neq 0$. Assume the noise term $\epsilon_{k,t}$ follows a zero-mean generalized normal distribution:
\begin{equation}
p(\epsilon_{k,t}) \propto e^{-\lambda |\epsilon_{k,t}|^\beta}\textrm{,~~with positive $\lambda$ and }\beta > 2 \textrm{ and } \beta \neq 3.
\end{equation}
If Eq.~\ref{eq:invert} holds, then $\mathbf{z}_{t}$ must be an invertible, component-wise transformation of permuted  $\hat{\mathbf{z}}_t$.

\end{corollary}

\subsection{Further Benefits from Changing Causal Influences}\label{sec:change}

Let $\eta_{kt}(u_r)  \triangleq \log p(z_{k, t} | \mathbf{z}_{t-1}, u_r)$ where $r=1,...,m$. LEAP \cite{yao2021learning} established the identifiability of the latent temporal causal processes $\mathbf{z}_t$ in certain nonstationary cases, under the condition that the noise term in each $z_{k,t}$, relative to its parents in $\mathbf{z}_{t-1}$, changes across $m$ contexts corresponding to $\mathbf{u} = u_1, u_2, ..., u_m$. Here we show that the identifiability result shown in the previous section can further benefit from nonstationarity of the causal model, and that our identifiability condition is generally much weaker than that in \cite{yao2021learning}: we allow changes in the noise term or causal influence on $z_{k,t}$ from its parents in $\mathbf{z}_{t-1}$, and our ``sufficient variability" condition is just a necessary condition for that in \cite{yao2021learning} because of the additional information that one can leverage.
Let $\mathbf{v}_{k,t}(u_r)$ be $\mathbf{v}_{k,t}$, which is defined in Eq.~\ref{eq:v}, in the $u_r$ context. Similarly, Let $\mathring{\mathbf{v}}_{k,t}(u_r)$ be $\mathring{\mathbf{v}}_{k,t}$ in the $u_r$ context. Let 
\begin{equation}\label{eq:vns}
\setlength{\abovedisplayskip}{1pt}
\setlength{\belowdisplayskip}{1pt}
\resizebox{.465\hsize}{!}{
$\mathbf{s}_{k,t} \triangleq \Big( \mathbf{v}_{k,t}(u_1)^\intercal, ..., 
\mathbf{v}_{k,t}(u_m)^\intercal, 
\Delta^2_{2}, ...,
 \Delta^2_{m}
 \Big)^\intercal$},
\resizebox{.465\hsize}{!}{
$\mathring{\mathbf{s}}_{k,t} \triangleq \Big( \mathring{\mathbf{v}}_{k,t}(u_1)^\intercal, ..., 
\mathring{\mathbf{v}}_{k,t}(u_m)^\intercal, 
\Delta_{2}, ...,
 \Delta_{m}
 \Big)^\intercal,$}
\end{equation}

where  
$\Delta^2_{r}= \frac{\partial^2 \eta_{kt}({u}_{r})}{\partial z_{k,t}^2 } - 
 \frac{\partial^2 \eta_{kt}({u}_{r-1})}{\partial z_{k,t}^2 } $ and $\Delta_{r}= \frac{\partial \eta_{kt}({u}_{r})}{\partial z_{k,t} } - 
 \frac{\partial \eta_{kt}({u}_{r-1})}{\partial z_{k,t} }.$
As provided below, in our case, the identifiablity of $\mathbf{z}_t$ is guaranteed by the linear independence of the whole function vectors $\mathbf{s}_{k,t}$ and $\mathring{\mathbf{s}}_{k,t}$, with $k=1,2,...,n$.
However, the identifiability result in \cite{yao2021learning} relies on the linear independence of only the last $m-1$ components of $\mathbf{s}_{k,t}$ and $\mathring{\mathbf{s}}_{k,t}$ with $k=1,2,...,n$; this linear independence is generally a much stronger and restricted condition.

\begin{theorem}[Identifiability under Changing Causal Dynamics]
Suppose $\mathbf{x}_t = \mathbf{g}(\mathbf{z}_t)$ and that the conditional distribution $p(z_{k,t} \,|\, \mathbf{z}_{t-1})$ may change across $m$ values of the context variable $\mathbf{u}$, denoted by $u_1$, $u_2$, ..., $u_m$. Suppose the components of $\mathbf{z}_t$ are mutually independent conditional on $\mathbf{z}_{t-1}$ in each context. Assume that the components of $\hat{\mathbf{z}}_t$ produced by Eq.~\ref{eq:invert} are also mutually independent conditional on $\hat{\mathbf{z}}_{t-1}$. 
If the $2n$ function vectors $\mathbf{s}_{k,t}$ and $\mathring{\mathbf{s}}_{k,t}$, with $k=1,2,...,n$, are linearly independent, then $\hat{\mathbf{z}}_t$ is a permuted invertible component-wise transformation of $\mathbf{z}_t$. 
\end{theorem}

\begin{theorem}[Identifiability under Observation Changes]
Suppose $\mathbf{x}_t = \mathbf{g}(\mathbf{z}_t)$ and that the conditional distribution $p(z_{k,t} \,|\, \mathbf{u})$ may change across $m$ values of the context variable $\mathbf{u}$, denoted by $u_1$, $u_2$, ..., $u_m$. Suppose the components of $\mathbf{z}_t$ are mutually independent conditional on $\mathbf{u}$ in each context. Assume that the components of $\hat{\mathbf{z}}_t$ produced by Eq.~\ref{eq:invert} are also mutually independent conditional on $\hat{\mathbf{z}}_{t-1}$. 
If the $2n$ function vectors $\mathbf{s}_{k,t}$ and $\mathring{\mathbf{s}}_{k,t}$, with $k=1,2,...,n$, are linearly independent, then $\hat{\mathbf{z}}_t$ is a permuted invertible component-wise transformation of $\mathbf{z}_t$.
\end{theorem}

\begin{corollary}[Identifiability under Modular Distribution Shifts] Assume the data generating process in Eq.~\ref{eq:model}. If the three partitioned latent components $\mathbf{z}_t = (\mathbf{z}_t^{\text{fix}}, \mathbf{z}_t^{\text{chg}}, \mathbf{z}_t^{\text{obs}})$ respectively satisfy the conditions in \textbf{Theorem} 1, \textbf{Theorem} 2, and \textbf{Theorem} 3, then $\mathbf{z}_{t}$ must be an invertible, component-wise transformation of a permuted version of $\hat{\mathbf{z}}_t$. 
\end{corollary}

\section{Our Approach}\label{sec:method}
\subsection{\ourmeos: Temporally Disentangled Representation Learning}
Given our identifiability results, we propose \ourmeos framework to estimate the latent causal dynamics under modular distribution shifts, by extending Sequential Variational Auto-Encoders \cite{li2018disentangled} with tailored modules to model different distribution shifts, and enforcing the conditions
in Sec.~\ref{sec:theory}  as constraints. We give the estimation procedure of the latent causal dynamics model in Eq.~\ref{eq:model}. The model architecture is showcased in Fig.~\ref{fig:arch}. The framework has the following three major components. The implementation details are in Appendix~\ref{ap:arch}. Specifically, we leverage the partitioned estimated latent subspaces $\hat{\mathbf{z}}_t = (\hat{\mathbf{z}}_t^{\text{fix}}, \hat{\mathbf{z}}_t^{\text{chg}}, \hat{\mathbf{z}}_t^{\text{obs}})$ and model their distribution changes in conditional transition priors. We use $[f_s, f_c, f_o]$ to capture causal relations as in Eq.~\ref{eq:model}, where $f_s$ captures fixed causal influences, $f_c$ for changing causal influences, and $f_o$ for observation changes. Accordingly, we learn $[\hat{f}_s^{-1}, \hat{f}_c^{-1}, \hat{f}_o^{-1}]$ to output random process noise from the estimated direct cause (lagged states $\hat{\mathbf{z}}_{\text{Hx}}$) and effect (current states) variables $\hat{\mathbf{z}}_t$. 
\begin{wrapfigure}{r}{8.5cm}
\centering
\includegraphics[width=\linewidth]{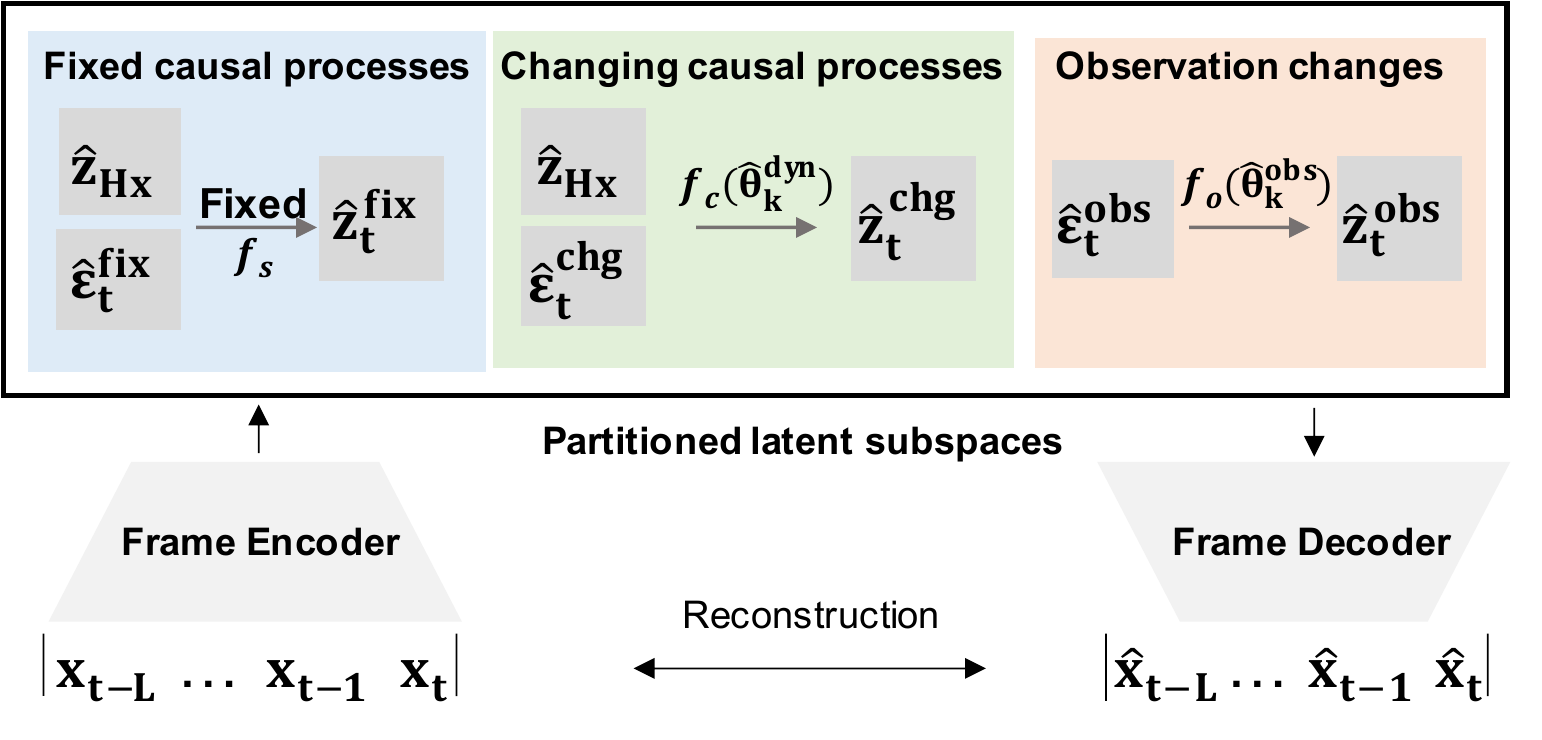}
\caption{\ourmeos describes each domain with change factors $(\hat{\boldsymbol{\theta}}_r^{\text{dyn}}, \hat{\boldsymbol{\theta}}_r^{\text{obs}})$ and inserts them into the prior models of the partitioned latent processes. The posteriors of the latent variables are inferred from image frames with variational autoencoder.}
\label{fig:arch}
\vspace{-1cm}
\end{wrapfigure}
\vspace{-0.5cm}
\paragraph{Change Factor Representation}
We learn to embed domain index $\mathbf{u}_r$ into low-dimensional change factors $(\hat{\boldsymbol{\theta}}_r^{\text{dyn}}, \hat{\boldsymbol{\theta}}_r^{\text{obs}})$ in Fig.~\ref{fig:arch} and insert them as external inputs to the (inverse) dynamics function $\hat{f}_c^{-1}(\hat{\boldsymbol{\theta}}_r^{\text{dyn}})$, or the observation bijector $\hat{f}_o^{-1}(\hat{\boldsymbol{\theta}}_r^{\text{obs}})$, respectively. And hence the distribution shifts are captured and utilized in the implementation.
\vspace{-0.25cm}
\paragraph{Modular Prior Network}
We follow standard conditional normalizing flow formulation \cite{dinh2016density,yao2021learning}. Let $\mathbf{z}_{\text{Hx}}$ denote the lagged latent variables up to maximum time lag $L$. In particular, for \textbf{1) fixed causal dynamics processes} $\hat{\mathbf{z}}_t^{\text{fix}}$, their transition priors  are obtained by first learning inverse transition functions $f_s^{-1}$ that take the estimated latent variables and output random noise terms, and applying the change of variables formula to the transformation: $p(\hat{z}_{s,t}^{\text{fix}} \vert \hat{\mathbf{z}}_{\text{Hx}}) = p_{\epsilon_s}\left(\hat{f}_s^{-1}(\hat{z}_{s,t}^{\text{fix}}, \hat{\mathbf{z}}_{\text{Hx}})\right)\Big|\frac{\partial \hat{f}_s^{-1}}{\partial \hat{z}_{s,t}^{\text{fix}}}\Big|$; for \textbf{2) changing causal dynamics}, we evaluate $p(\hat{z}_{c,t}^{\text{chg}} \vert \hat{\mathbf{z}}_{\text{Hx}}, \mathbf{u}_r) = p_{\epsilon_c}\left(\hat{f}_c^{-1}(\hat{z}_{c,t}^{\text{chg}}, \hat{\mathbf{z}}_{\text{Hx}}, \hat{\boldsymbol{\theta}}_r^{\text{dyn}})\right)\Big|\frac{\partial \hat{f}_c^{-1}}{\partial \hat{z}_{c,t}^{\text{chg}}}\Big|$ by learning a holistic inverse dynamics $\hat{f}_c^{-1}$ that takes the estimated change factors for dynamics $\hat{\boldsymbol{\theta}}_r^{\text{dyn}}$ as inputs, and similarly for \textbf{3) global observation changes} $\hat{\mathbf{z}}_t^{\text{obs}}$, we learn to project them to invariant noise terms by $\hat{f}_o^{-1}$ which takes the change factors $\boldsymbol{\theta}_r^{\text{obs}}$ as arguments, and obtains  $p(\hat{z}_{o,t}^{\text{obs}} \vert \mathbf{u}_r) = p_{\epsilon_o}\left(\hat{f}_o^{-1}(\hat{z}_{o,t}^{\text{obs}}, \hat{\boldsymbol{\theta}}_r^{\text{obs}})\right)\Big|\frac{\partial \hat{f}_o^{-1}}{\partial \hat{z}_{o,t}^{\text{obs}}}\Big|$ as the prior. \textbf{Conditional independence of the estimated latent variables} $p(\hat{\mathbf{z}}_t | \hat{\mathbf{z}}_{\text{Hx}})$ is enforced by summing up all estimated component densities when obtaining the joint $p(\mathbf{z}_t \vert \mathbf{z}_{\text{Hx}}, \mathbf{u})$ in Eq.~\ref{eq:np-trans}. Given that the Jacobian is lower-triangular, we can compute its determinant as the product of diagonal terms. The detailed derivations are given in Appendix~\ref{ap:derive}.
\begin{equation}\label{eq:np-trans}
\setlength{\abovedisplayskip}{1pt}
\setlength{\belowdisplayskip}{1pt}
\log p\left(\hat{\mathbf{z}}_t \vert \hat{\mathbf{z}}_{\text{Hx}}, \mathbf{u}_r\right) = \underbrace{\sum_{i=1}^n \log p(\hat{\epsilon_i} \vert \mathbf{u}_r)}_{\text{Conditional indepdence}}+ \underbrace{\sum_{i=1}^n \log \Big| \frac{\partial \hat{f}^{-1}_i}{\partial \hat{z}_{it}}\Big|}_{\text{Lower-triangular Jacobian}}  
\end{equation}
\vspace{-0.5cm}
\paragraph{Factorized Inference} We infer the posteriors of each time step $q(\hat{\mathbf{z}}_t \vert \mathbf{x}_t)$ using only the observation at that time step, because in Eq.~\ref{eq:model}, $\mathbf{x}_t$ preserves all the information of the current system states so the joint probability $q(\mathbf{\hat z}_{1:T}|\mathbf{x}_{1:T})$ can be factorized into product of these terms. We approximate the posterior $q(\mathbf{\hat z}_t|\mathbf{x}_t)$ with an isotropic Gaussian with mean and variance  from the inference network.
\paragraph{Optimization} We train \ourmeos using the ELBO objective $\mathcal{L}_{\text{ELBO}} = 
\frac{1}{N} \sum_{i \in N} \mathcal{L}_{\text{Recon}} - \beta \mathcal{L}_{\text{KLD}}$, in which we use mean-squared error (MSE) for the reconstruction likelihood $\mathcal{L}_{\text{Recon}}$.  The KL divergence  $\mathcal{L}_{\text{KLD}}$ is estimated via a sampling approach since with a learned nonparametric modular transition prior, the distribution does not have an explicit form. Specifically, we obtain the log-likelihood of the posterior, evaluate the prior $\log p\left(\hat{\mathbf{z}}_t \vert \hat{\mathbf{z}}_{\text{Hx}}, \mathbf{u}_r\right)$ in \Eqref{eq:np-trans}, and compute their mean difference in the dataset as the KL loss: $\mathcal{L}_{\text{KLD}} = \mathbb{E}_{\mathbf{\hat z}_t \sim q\left(\mathbf{\hat z}_t^{(i)} \vert \mathbf{x}_t^{(i)}\right)} \log q(\mathbf{\hat z}_t|\mathbf{x}_t) - \log p(\hat{\mathbf{z}}_t \vert \hat{\mathbf{z}}_{\text{Hx}}, \mathbf{u}_r)$.
\vspace{-0.3cm}
\subsection{Causal Visualization}
For visualization purposes, when the underlying latent processes have sparse causal relations, we  fit LassoNet \cite{lemhadri2021lassonet} on the latent processes recovered by \ourmeos to interpret the causal relations. Specifically, we fit LassoNet to predict $\hat{\mathbf{z}}_t$ using the estimated history information $\hat{\mathbf{z}}_{\text{Hx}} = \{\hat{\mathbf{z}}_{t-\tau}\}_{\tau=1}^{L}$ up to maximum time lag $L$.
Note that this postprocessing step is optional – the latent causal relations have already been captured in the learned transition functions in \ourmeos. Also, our identifiability conditions do not rely on the sparsity of causal relations in the latent processes.
%
%
%
%
%
%

\section{Experiments}
We evaluate the identifiability results of \ourmeos on a number of simulated and real-world time-series datasets. We first introduce the evaluation metrics and baselines. \textbf{(1) Evaluation Metrics.} To evaluate the identifiability of the latent variables, we compute Mean Correlation Coefficient (MCC) on the test dataset. MCC is a standard metric in the ICA literature for continuous variables which measure the identifiability of the learned latent causal processes. MCC is close to 1 when latent variables are  identifiable up to permutation and component-wise invertible transformation in the noiseless case. \textbf{(2) Baselines.} Nonlinear ICA methods are used: \textbf{(1)} BetaVAE \cite{higgins2016beta} which ignores both history and nonstationarity information;  \textbf{(2)} iVAE \cite{khemakhem2020variational} and TCL \cite{hyvarinen2016unsupervised} which leverage \underline{nonstationarity} to establish identifiability but assumes \underline{independent factors}, and \textbf{(3)} SlowVAE \cite{klindt2020towards} and PCL \cite{hyvarinen2017nonlinear} which exploit \underline{temporal constraints} but assume \underline{independent sources} and \underline{stationary processes}, and \textbf{(4)} LEAP \cite{yao2021learning} which assumes \underline{nonstationary}, \underline{causal processes} but only models   \underline{nonstationary noise}. Two other disentangled deep state-space models with nonlinear dynamics models: Kalman VAE (KVAE \cite{fraccaro2017disentangled}) and Deep Variational Bayes Filters (DVBF \cite{karl2016deep}), are also used for comparisons.

\vspace{-0.45cm}
\begin{table}[ht]
\centering
\caption{MCC scores and their standard deviations for the three simulation settings over 3 random seeds. Note: The symbol ``--'' represents that this method is not applicable to this dataset.}
\label{tab:syn-results}
\resizebox{\textwidth}{!}{%
\begin{tabular}{@{}llllllllll@{}}
\toprule
\multirow{2}{*}{\textbf{\begin{tabular}[c]{@{}c@{}}Experiment\\ Settings\end{tabular}}} & \multicolumn{9}{c}{\textbf{Method}}                               \\ \cmidrule(l){2-10} 
                                                                                        & \ourmeos & LEAP & SlowVAE & PCL & i-VAE & TCL & betaVAE & KVAE & DVBF \\ \midrule
Fixed &
  \textbf{0.954 \textpm 0.009} &
  -- &
  0.411 \textpm 0.022 &
  0.516 \textpm 0.043 &
  -- &
  -- &
  0.353 \textpm 0.001 &
  \underline{0.832 \textpm 0.038} &
  0.778 \textpm 0.045
   \\
Changing &
  \textbf{0.958 \textpm 0.017} &
  \underline{0.726 \textpm 0.187} &
  0.511 \textpm 0.062 &
  0.599 \textpm 0.041 &
  0.581 \textpm 0.083 &
  0.399 \textpm 0.021 &
  0.523 \textpm 0.009 &
  \underline{0.711 \textpm 0.062} &
  0.648 \textpm 0.071
   \\
Modular &
  \textbf{0.993 \textpm 0.001} &
  \underline{0.657 \textpm 0.108} &
  0.406 \textpm 0.045 &
  0.564 \textpm 0.049 &
  0.557 \textpm 0.005 &
  0.297 \textpm 0.078 &
  0.433 \textpm 0.045 &
  0.632 \textpm 0.048 &
  0.678 \textpm 0.074
   \\ \bottomrule
\end{tabular}%
}
\end{table}
\subsection{Simulated Experiments}
We generate synthetic datasets that satisfy our identifiability conditions in the theorems following the procedures in Appendix~\ref{ap:synthetic}. As in Table~\ref{tab:syn-results}, our framework can recover the latent processes under fixed dynamics (heterogeneous noise model), under changing causal dynamics, and under modular distribution shifts with high MCCs (>0.95). The baselines that do not exploit history (i.e., $\beta$VAE, i-VAE, TCL), with independent source assumptions (SlowVAE, PCL), considers limited nonstationary cases (LEAP) distort the identifiability results. KVAE and DVBF achieve MCCs (0.8) under fixed dynamics but distorts the identifiability under changing dynamics and modular shift setting because they don't model changing causal relations and global observation changes. 

\subsection{Real-world Applications}


\paragraph{Video Data -- Modified Cartpole Environment}

We evaluate \ourmeos on the modified cartpole \cite{huang2021adarl} video dataset and compare the performances with the baselines. Modified Cartpole is a nonlinear dynamical system with cart positions $x_t$ and pole angles $\theta_t$ as the true state variables. The dataset descriptions are in Appendix \ref{ap:real}. We use 6 source domains with different gravity values $g= \{5, 10, 15, 20, 25, 30\}$. Together with the 2 discrete actions (i.e., left and right), we have 12 segments of data with changing causal dynamics. We fit \ourmeos with two-dimensional change factors $\boldsymbol{\theta}_r^{\text{dyn}}$. We set the latent size $n=8$ and the lag number $L=2$. In Fig.~\ref{fig:cartpole}, the latent causal processes are recovered, as seen from (a) high MCC for the latent causal processes; (b) the latent factors are estimated up to component-wise transformation; (c) \ourmeos outperforms the baselines and (d) the latent traversals confirm the two recovered latent variables correspond to the position and pole angle.

\begin{figure}[ht]
    \centering
    \includegraphics[width=\linewidth]{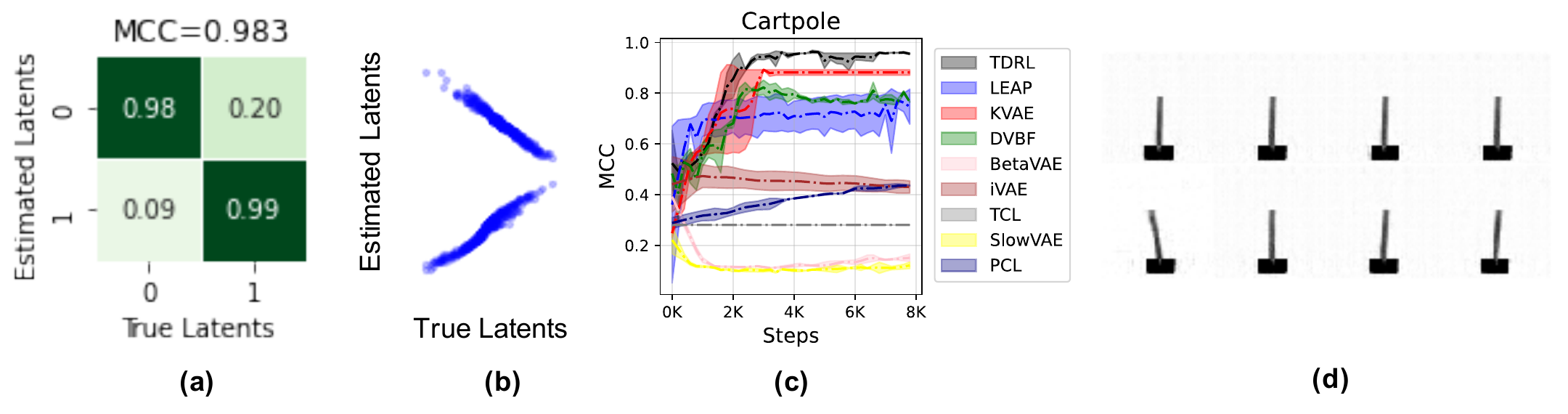}
    \caption{Modified Cartpole results: (a) MCC for causally-related factors; (b) scatterplots between estimated and true factors; (c) baseline comparisons, and (d) latent traversal on a fixed video frame;.  \label{fig:cartpole}}
\end{figure}

\paragraph{Motion Capture Data -- CMU-Mocap}
We experimented with another real-world motion capture dataset (CMU-Mocap). The dataset descriptions are in Appendix \ref{ap:real}. We fit \ourmeos with 11 trials of motion capture data for subject \#8. The 11 trials contain walk cycles with very different styles (e.g., slow walk, stride). We set latent size $n=3$ and lag number $L=2$. The differences between trials are captured through learning the 2-dimensional change factors for each trial. In Fig.~\ref{fig:mocap}(a), the learned change factors group similar walk styles into clusters; in Panel (c), three latent variables (which seem to be pitch, yaw, roll rotations) are found to explain most of the variances of human walk cycles. The learned latent coordinates (Panel b) show smooth cyclic patterns with differences among different walking styles. For the discovered skeleton (Panel d), roll and pitch of walking are found to be causally-related while yaw has independent dynamics. 

\begin{figure}[ht]
    \centering
    \includegraphics[width=\linewidth]{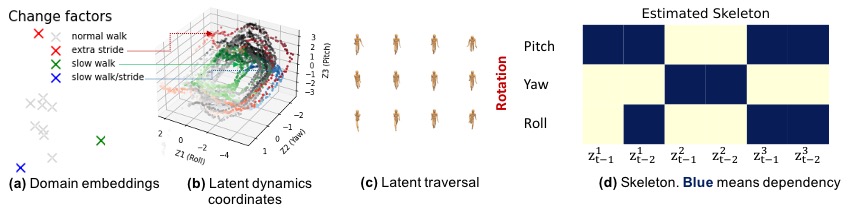}
    \caption{CMU-Mocap results (Subject \#8): (a) learned change factors; (b) latent coordinates dynamics for 11 trials; (c) ) latent traversal by rendering the reconstructed point clouds into the video frame; (d) estimated causal skeleton (blue indicates dependency). \label{fig:mocap}}
\end{figure}

\section{Conclusion}

In this paper, without relying on parametric or distribution assumptions, we established identifiability theories for nonparametric latent causal processes from their observed nonlinear mixtures in stationary environments and under unknown distribution shifts. \textbf{The basic limitation} of this work is that the underlying latent processes are assumed to have no instantaneous causal relations but only time-delayed influences. If the time resolution of the observed time series is the same as the causal frequency (or even higher), this assumption naturally holds. However, if the resolution of the time series is much lower, then it is usually violated and one has to find a way to deal with instantaneous causal relations. On the other hand, it is worth mentioning that the assumptions are generally testable; if the assumptions are actually violated, one can see that the results produced by our method may not be reliable. Extending our theories and framework to address the issue of instantaneous dependency or instantaneous causal relations will be one line of our future work. In addition, Empirically exploration of the merits of the identified latent causal graph in terms of few-shot transfer to new environments \cite{yao2022learning}, domain adaptation \cite{pmlr-v162-kong22a}, forecasting \cite{huang2019causal,li2021transferable} and control \cite{huang2021adarl} is also one important future step. 

\section*{Acknowledgment}
\vspace{-0.25cm}

 Kun Zhang was partially supported by the National Institutes of Health (NIH) under Contract R01HL159805, by the NSF-Convergence Accelerator Track-D award \#2134901, by a grant from Apple Inc., and by a grant from KDDI Research Inc.

    

\bibliographystyle{unsrt}
\bibliography{references}

\clearpage

\clearpage
\appendix

  \textit{\large Supplement to}\\ \ \\
      {\large \bf ``\ourtitle''}\
\vspace{.2cm}

	
\newcommand{\beginsupplement}{%
	\setcounter{table}{0}
	\renewcommand{\thetable}{S\arabic{table}}%
	\setcounter{figure}{0}
	\renewcommand{\thefigure}{S\arabic{figure}}%
	\setcounter{section}{0}
	\renewcommand{\thesection}{S\arabic{section}}%
	\setcounter{theorem}{0}
	\renewcommand{\thetheorem}{S\arabic{theorem}}%
	\setcounter{proposition}{0}
	\renewcommand{\theproposition}{S\arabic{proposition}}%
	\setcounter{corollary}{0}
	\renewcommand{\thecorollary}{S\arabic{corollary}}%
	\renewcommand{\thelemma}{S\arabic{lemma}}%
}

\beginsupplement

{\large Appendix organization:}

\DoToC



\section{Identifiability Theory}\label{ap:theory}

The observed variables were generated according to : 
\begin{equation} \label{Eq:generation}
\mathbf{x}_t = \mathbf{g}(\mathbf{z}_t),    
\end{equation}
in which  $\mathbf{g}$ is invertible, and $z_{it}$, as the $i$th component of  $\mathbf{z}_t$, is generated by (some) components of $\mathbf{z}_{t-1}$ and noise $E_{it}$. $E_{1t}, E_{2t}, ..., E_{nt}$ are mutually independent. In other words, the components of $\mathbf{z}_{t}$ are mutually independent conditional on $\mathbf{z}_{t-1}$. Let $\eta_{kt} \triangleq \log p(z_{kt} | \mathbf{z}_{t-1})$.  Assume that $\eta_k(t)$ is twice differentiable in $z_{kt}$ and is differentiable in $z_{l,t-1}$, $l=1,2,...,n$. Note that the parents of $z_{kt}$ may be only a subset of $\mathbf{z}_{t-1}$; if $z_{l,t-1}$ is not a parent of $z_{kt}$, then $\frac{\partial \eta_k}{\partial z_{l,t-1}} = 0$.

\subsection{Proof for Theorem 1}

\begin{theorem} [Identifiablity under a Fixed Temporal Causal Model] \label{Theo1} 
Suppose there exists invertible function $\mathbf{f}$, which is the estimated mixing function (i.e., we use $\mathbf{f}$ and $\hat{\mathbf{g}}$ interchangeably in Appendix) that maps $\mathbf{x}_t$ to $\hat{\mathbf{z}}_t$, i.e., 
\begin{equation} \label{Eq:invert}
    \hat{\mathbf{z}}_t = \mathbf{f}(\mathbf{x}_t)
\end{equation}
such that the components of $\hat{\mathbf{z}}_t$ are mutually  independent conditional on $\hat{\mathbf{z}}_{t-1}$. 
Let 
\begin{equation} \label{Eq:v}
\begin{aligned}
\mathbf{v}_{k,t} &\triangleq \Big(\frac{\partial^2 \eta_{kt}}{\partial z_{k,t} \partial z_{1,t-1}}, \frac{\partial^2 \eta_{kt}}{\partial z_{k,t} \partial z_{2,t-1}}, ..., \frac{\partial^2 \eta_{kt}}{\partial z_{k,t} \partial z_{n,t-1}} \Big)^\intercal \\ \mathring{\mathbf{v}}_{k,t} &\triangleq \Big(\frac{\partial^3 \eta_{kt}}{\partial z_{k,t}^2 \partial z_{1,t-1}}, \frac{\partial^3 \eta_{kt}}{\partial z_{k,t}^2 \partial z_{2,t-1}}, ..., \frac{\partial^3 \eta_{kt}}{\partial z_{k,t}^2 \partial z_{n,t-1}} \Big)^\intercal. 
\end{aligned}
\end{equation}
If for each value of $\mathbf{z}_t$, $\mathbf{v}_{1t}, \mathring{\mathbf{v}}_{1t}, \mathbf{v}_{2t}, \mathring{\mathbf{v}}_{2t}, ..., \mathbf{v}_{nt}, \mathring{\mathbf{v}}_{nt}$, as 
$2n$ vector functions in $z_{1,t-1}$, $z_{2,t-1}$, ..., $z_{n,t-1}$, are linearly independent, then $\mathbf{z}_{t}$ must be an invertible, component-wise transformation of a permuted version of $\hat{\mathbf{z}}_t$.
\end{theorem}

\begin{proof}
Combining Eq.~\ref{Eq:generation} and Eq.~\ref{Eq:invert} gives $\mathbf{z}_t = \mathbf{g}^{-1}(\mathbf{f}^{-1}(\hat{\mathbf{z}}_t)) = \mathbf{h}(\hat{\mathbf{z}}_t)$, where $\mathbf{h} \triangleq \mathbf{g}^{-1}\circ \mathbf{f}^{-1}$. 
 Since both $\mathbf{f}$ and $\mathbf{g}$ are invertible, $\mathbf{h}$ is invertible.  Let $\mathbf{H}_t$ be the Jacobian matrix of the transformation $h(\hat{\mathbf{z}}_t)$, and denote by $\mathbf{H}_{kit}$ its $(k,i)$th entry.
 
 First, it is straightforward to see that if the components of $\hat{\mathbf{z}}_t$ are mutually independent conditional on $\hat{\mathbf{z}}_{t-1}$, then for any $i\neq j$, $\hat{z}_{it}$ and $\hat{z}_{jt}$ are conditionally independent given $\hat{\mathbf{z}}_{t-1} \cup(\hat{\mathbf{z}}_{t}\setminus \{\hat{z}_{it}, \hat{z}_{jt}\})$. Mutual independence of the components of $\hat{\mathbf{z}}_t$ conditional on $\hat{\mathbf{z}}_{t-1}$ implies that $\hat{z}_{it}$ is independent from $\hat{\mathbf{z}}_{t}\setminus \{\hat{z}_{it}, \hat{z}_{jt}\}$ conditional on $\hat{\mathbf{z}}_{t-1}$, i.e.,
 $$p(\hat{z}_{it} \,|\, \hat{\mathbf{z}}_{t-1}) = p(\hat{z}_{it} \,|\, \hat{\mathbf{z}}_{t-1}\cup(\hat{\mathbf{z}}_{t}\setminus \{\hat{z}_{it}, \hat{z}_{jt}\})).$$
 At the same time, it also implies $\hat{z}_{it}$ is independent from $\hat{\mathbf{z}}_{t}\setminus \{\hat{z}_{it}\}$ conditional on $\hat{\mathbf{z}}_{t-1}$, i.e., 
 $$p(\hat{z}_{it} \,|\, \hat{\mathbf{z}}_{t-1}) = p(\hat{z}_{it} \,|\, \hat{\mathbf{z}}_{t-1}\cup(\hat{\mathbf{z}}_{t}\setminus \{\hat{z}_{it}\})).$$
 Combining the above two equations gives $p(\hat{z}_{it} \,|\, \hat{\mathbf{z}}_{t-1}\cup(\hat{\mathbf{z}}_{t}\setminus \{\hat{z}_{it}\})) = p(\hat{z}_{it} \,|\,\hat{\mathbf{z}}_{t-1} \cup(\hat{\mathbf{z}}_{t}\setminus \{\hat{z}_{it}, \hat{z}_{jt}\}))$, i.e., for $i\neq j$, $\hat{z}_{it}$ and $\hat{z}_{jt}$ are conditionally independent given $\hat{\mathbf{z}}_{t-1} \cup(\hat{\mathbf{z}}_{t}\setminus \{\hat{z}_{it}, \hat{z}_{jt}\})$.
 
 We then make use of the fact that if  $\hat{z}_{it}$ and $\hat{z}_{jt}$ are conditionally independent given $\hat{\mathbf{z}}_{t-1} \cup(\hat{\mathbf{z}}_{t}\setminus \{\hat{z}_{it}, \hat{z}_{jt}\})$, then 
 $$\frac{\partial^2 \log p(\hat{\mathbf{z}}_t, \hat{\mathbf{z}}_{t-1})}{\partial \hat{z}_{it} \partial \hat{z}_{jt}} = 0, $$ 
 assuming the cross second-order derivative exists \cite{Spantini2018InferenceVL}. Since $p(\hat{\mathbf{z}}_t, \hat{\mathbf{z}}_{t-1}) = p(\hat{\mathbf{z}}_t \,|\, \hat{\mathbf{z}}_{t-1})p(\hat{\mathbf{z}}_{t-1})$ while $p(\hat{\mathbf{z}}_{t-1})$ does not involve $\hat{z}_{it}$ or $\hat{z}_{jt}$, the above equality is equivalent to 
 \begin{equation} \label{Eq:iszero}
     \frac{\partial^2 \log p(\hat{\mathbf{z}}_t \,|\,\hat{\mathbf{z}}_{t-1})}{\partial \hat{z}_{it} \partial \hat{z}_{jt}} = 0.
 \end{equation}
 
 The Jacobian matrix of the mapping from $(\mathbf{x}_{t-1}, \hat{\mathbf{z}}_t)$ to $(\mathbf{x}_{t-1}, \mathbf{z}_t)$ is $\begin{bmatrix}\mathbf{I} & \mathbf{0} \\ * & \mathbf{H}_t \end{bmatrix}$, where $*$ stands for a matrix, and the (absolute value of the) determinant of this Jacobian matrix is $|\mathbf{H}_t|$. Therefore $p(\hat{\mathbf{z}}_t, \mathbf{x}_{t-1}) = p({\mathbf{z}}_t, \mathbf{x}_{t-1})\cdot |\mathbf{H}_t|$. Dividing both sides of this equation by $p(\mathbf{x}_{t-1})$ gives 
 \begin{equation} \label{Eq:J_trans}
 p(\hat{\mathbf{z}}_t \,|\, \mathbf{x}_{t-1}) = p({\mathbf{z}}_t \,|\, \mathbf{x}_{t-1}) \cdot |\mathbf{H}_t|. 
 \end{equation}
 Since $p({\mathbf{z}}_t \,|\, {\mathbf{z}}_{t-1}) = p({\mathbf{z}}_t \,|\, \mathbf{g}({\mathbf{z}}_{t-1})) = p({\mathbf{z}}_t \,|\, {\mathbf{x}}_{t-1})$ and similarly   $p(\hat{\mathbf{z}}_t \,|\, \hat{\mathbf{z}}_{t-1}) = p(\hat{\mathbf{z}}_t \,|\, {\mathbf{x}}_{t-1})$, Eq.~\ref{Eq:J_trans} tells us
 \begin{equation}
     \log p(\hat{\mathbf{z}}_t \,|\, \hat{\mathbf{z}}_{t-1}) = \log p({\mathbf{z}}_t \,|\, {\mathbf{z}}_{t-1}) + \log |\mathbf{H}_t| = \sum_{k=1}^n \eta_{kt} + \log |\mathbf{H}_t|.
 \end{equation}
 Its partial derivative w.r.t. $\hat{z}_{it}$ is 
 \begin{flalign} \nonumber
  \frac{\partial \log p(\hat{\mathbf{z}}_t \,|\, \hat{\mathbf{z}}_{t-1})}{\partial \hat{z}_{it}} &=  \sum_{k=1}^n \frac{\partial \eta_{kt} }{\partial z_{kt}} \cdot \frac{\partial z_{kt}}{\partial \hat{z}_{it}} - \frac{\partial \log |\mathbf{H}_t|}{\partial \hat{z}_{it}} \\ \nonumber
  &= \sum_{k=1}^n \frac{\partial \eta_{kt}}{\partial z_{kt}} \cdot \mathbf{H}_{kit} - \frac{\partial \log |\mathbf{H}_t|}{\partial \hat{z}_{it}}.
 \end{flalign}
  Its second-order cross derivative is
 \begin{flalign} \label{Eq:cross}
  \frac{\partial^2 \log p(\hat{\mathbf{z}}_t \,|\, \hat{\mathbf{z}}_{t-1})}{\partial \hat{z}_{it} \partial \hat{z}_{jt}}
  &= \sum_{k=1}^n \Big( \frac{\partial^2 \eta_{kt}}{\partial z_{kt}^2 } \cdot \mathbf{H}_{kit}\mathbf{H}_{kjt} + \frac{\partial \eta_{kt}}{\partial z_{kt}} \cdot \frac{\partial \mathbf{H}_{kit}}{\partial \hat{z}_{jt}} \Big)- \frac{\partial^2 \log |\mathbf{H}_t|}{\partial \hat{z}_{it} \partial \hat{z}_{jt}}.
 \end{flalign}
 
 The above quantity is always 0 according to Eq.~\ref{Eq:iszero}. Therefore, for each $l=1,2,...,n$ and each value $z_{l,t-1}$,  its partial derivative w.r.t.
 $z_{l,t-1}$ is always 0. That is,
 \begin{flalign}\label{eq:lind-ap}
  \frac{\partial^3 \log p(\hat{\mathbf{z}}_t \,|\, \hat{\mathbf{z}}_{t-1})}{\partial \hat{z}_{it} \partial \hat{z}_{jt} \partial z_{l,t-1}}
  &= \sum_{k=1}^n \Big( \frac{\partial^3 \eta_{kt}}{\partial z_{kt}^2 \partial z_{l,t-1}} \cdot \mathbf{H}_{kit}\mathbf{H}_{kjt} + \frac{ \partial^2 \eta_{kt}}{\partial z_{kt} \partial z_{l,t-1}}  \cdot \frac{\partial \mathbf{H}_{kit}}{\partial \hat{z}_{jt} } \Big) \equiv 0,
 \end{flalign}
 where we have made use of the fact that entries of $\mathbf{H}_t$ do not depend on $z_{l,t-1}$. 
 
 If for any value of $\mathbf{z}_t$, $\mathbf{v}_{1t}, \mathring{\mathbf{v}}_{1t}, \mathbf{v}_{2t}, \mathring{\mathbf{v}}_{2t}, ..., \mathbf{v}_{nt}, \mathring{\mathbf{v}}_{nt}$ are linearly independent, to make the above equation hold true, one has to set $\mathbf{H}_{kit}\mathbf{H}_{kjt} = 0$ or $i\neq j$. That is, in each row of $\mathbf{H}_t$ there is only one non-zero entry. Since $h$ is invertible, then $\mathbf{z}_{t}$ must be an invertible, component-wise transformation of a permuted version of $\hat{\mathbf{z}}_t$.
\end{proof}

The linear independence condition in Theorem \ref{Theo1} is the core condition to guarantee the identifiability of $\mathbf{z}_t$ from the observed $\mathbf{x}_t$. Roughly speaking, for a randomly chosen conditional density function $p(z_{kt}\,|\,\mathbf{z}_{t-1})$, the chance for this constraint to hold on its second- and third-order partial derivatives is slim. For illustrative purposes, below we make this claim more precise, by considering a specific unidentifiable case, in which the noise terms in $\mathbf{z}_t$ are additive Gaussian, and two identifiable cases, in which $\mathbf{z}_t$ has additive, heterogeneous noise or follows some linear, non-Gaussian temporal process. 

Let us start with an unidentifiable case. 
If all $z_{kt}$ follow the additive noise model with Gaussian noise terms, i.e., 
\begin{equation} \label{Eq:Gaussian_case}
    \mathbf{z}_t = \mathbf{q}(\mathbf{z}_{t-1}) + \mathbf{E}_t,
\end{equation}
where $\mathbf{q}$ is a transformation and the components of the Gaussian vector $\mathbf{E}_t$ are independent and also independent from $\mathbf{z}_{t-1}$. Then  $\frac{\partial^2 \eta_{kt}}{\partial z_{kt}^2}$ is constant, and $\frac{\partial^3 \eta_{kt}}{\partial z_{kt}^2 \partial z_{l,t-1}} \equiv 0$, violating the linear independence condition. In the following proposition we give some alternative solutions and verify the unidentifiability in this case.

\begin{proposition} [Unidentifiability under Gaussian noise] 
Suppose $\mathbf{x}_t$ was generated according to Eq.~\ref{Eq:generation} and Eq.~\ref{Eq:Gaussian_case},
where the components of $\mathbf{E}_t$ are mutually independent Gaussian and also independent from $\mathbf{z}_{t-1}$. Then any $\hat{\mathbf{z}}_t = \mathbf{D}_1 \mathbf{U} \mathbf{D}_2 \cdot {\mathbf{z}}_t$, where $\mathbf{D}_1$ is an arbitrary non-singular diagonal matrix, $\mathbf{U}$ is an arbitrary orthogonal matrix, and $\mathbf{D}_2$ is a diagonal matrix with $\mathbb{V}ar^{-1/2}(E_{kt})$ as its $k$th  diagonal entry, is a valid solution to satisfy the condition that the components of $\hat{\mathbf{z}}_t$ are mutually independent conditional on $\hat{\mathbf{z}}_{t-1}$.
\end{proposition}
\begin{proof}
In this case we have 
$$ \hat{\mathbf{z}}_t = \mathbf{D}_1 \mathbf{U} \mathbf{D}_2 \cdot \mathbf{q}(\mathbf{z}_{t-1}) + \mathbf{D}_1 \mathbf{U} \mathbf{D}_2\cdot \mathbf{E}_t.$$
It is easy to verify that the components of $\mathbf{D}_1 \mathbf{U} \mathbf{D}_2\cdot \mathbf{E}_t$ are mutually independent and are independent from $\mathbf{D}_1 \mathbf{U} \mathbf{D}_2 \cdot \mathbf{q}(\mathbf{z}_{t-1})$. As a consequence, $\hat{\mathbf{z}}_t$ are mutually independent conditional on $\hat{\mathbf{z}}_{t-1}$.
\end{proof}

Now let us consider some cases in which the latent temporally processes $\mathbf{z}_t$ are naturally identifiable under some technical conditions. Let us first consider the case where $z_{kt}$ follows a heterogeneous noise process, in which the noise variance depends on its parents:
\begin{equation} \label{Eq:heteo}
    a_{kt} = q_k(\mathbf{z}_{t-1}) + \frac{1}{b_k(\mathbf{z}_{t-1})}E_{kt}.
\end{equation}
Here we assume $E_{kt}$ is standard Gaussian and $E_{1t}, E_{2t}, .., E_{nt}$ are mutually independent and independent from $\mathbf{z}_{t-1}$. $\frac{1}{b_k}$, which depends on $\mathbf{z}_{t-1}$, is the standard deviation of the noise in $z_{kt}$. (For conciseness, we drop the argument of $b_k$ and $q_k$ when there is no confusion.) Note that in this model, if $q_k$ is 0 for all $k=1,2,...,n$, it reduces to a multiplicative noise model. 
The identifiability result of $\mathbf{z}_t$ is established in the following proposition.

\begin{corollary} [Identifiablity under Heterogeneous Noise] 
Suppose $\mathbf{x}_t$ was generated according to Eq.~\ref{Eq:generation} and Eq.~\ref{Eq:heteo}. Suppose Eq.~\ref{Eq:invert} holds true. If  $b_k\cdot \frac{\partial b_k}{\partial \mathbf{z}_{t-1}}$ and $b_k\cdot \frac{\partial b_k}{\partial \mathbf{z}_{t-1}} (z_{kt} - q_{k}) - b_k^2\cdot \frac{\partial q_{k}}{\partial \mathbf{z}_{t-1}}$, with $k=1,2,...,n$, which are in total $2n$ function vectors in $\mathbf{z}_{t-1}$, 
are linearly independent, then $\mathbf{z}_{t}$ must be an invertible, component-wise transformation of a permuted version of $\hat{\mathbf{z}}_t$.

\end{corollary}

\begin{proof}
Under the assumptions, one can see that 
$$ \eta_{kt} = \log p(z_{kt}\,|\, \mathbf{z}_{t-1}) = -\frac{1}{2}\log(2\pi) + \log b_k - \frac{b_k^2}{2} (z_{kt} - q_{k} )^2.$$
Consequently, one can find
\begin{flalign}
 \nonumber
\frac{\partial^3 \eta_{kt}}{\partial z_{kt}^2 \partial z_{l,t-1}} &= -b_k\cdot \frac{\partial b_k}{\partial z_{l,t-1}}, \\ \nonumber
\frac{\partial^2 \eta_{kt}}{\partial z_{kt} \partial z_{l,t-1}} &= -b_k\cdot \frac{\partial b_k}{\partial z_{l,t-1}} (z_{kt} - q_{k}) + b_k^2\cdot \frac{\partial q_{k}}{\partial z_{l,t-1}}.
\end{flalign} 
Then the linear independence of $\mathbf{v}_{kt}$ and $\mathring{\mathbf{v}}_{kt}$ (defined in Eq.~\ref{Eq:v}), with $k=1.,2,...,n$, reduces to the linear independence condition in this proposition. Theorem \ref{Theo1}
 then implies that $\mathbf{z}_{t}$ must be an invertible, component-wise transformation of a permuted version of $\hat{\mathbf{z}}_t$.
\end{proof}

Let us then consider another special case, with linear, non-Gaussian temporal model for $\mathbf{z}_t$: the latent processes follow Eq.~\ref{Eq:Gaussian_case}, with $\mathbf{q}$ being a linear transformation and $E_{kt}$ following a particular class of non-Gaussian distributions. 
The following corollary shows that $\mathbf{z}_t$ is identifiable as long as each $z_{kt}$ receives causal influences from some components of $\mathbf{z}_{t-1}$.
\begin{corollary} [Identifiablity under a Specific Linear, Non-Gaussian Model for Latent Processes]  
Suppose $\mathbf{x}_t$ was generated according to Eq.~\ref{Eq:generation} and Eq.~\ref{Eq:Gaussian_case}, in which $\mathbf{q}$ is a linear transformation and for each $z_{kt}$, there exists at least one $k'$ such that $c_{kk'} \triangleq \frac{\partial z_{kt}}{\partial z_{k',t-1}} \neq 0$. Assume the noise term $E_{kt}$ follows a zero-mean generalized normal distribution:
\begin{equation}
    p(E_{kt}) \propto e^{-\lambda |e_{kt}|^\beta}\textrm{,~~with positive $\lambda$ and }\beta > 2 \textrm{ and } \beta \neq 3.
\end{equation}
Suppose Eq.~\ref{Eq:invert} holds true. Then $\mathbf{z}_{t}$ must be an invertible, component-wise transformation of a permuted version of $\hat{\mathbf{z}}_t$.
\end{corollary}

\begin{proof}
In this case, we have
\begin{flalign}
\frac{\partial^3 \eta_{kt}}{\partial z_{kt}^2 \partial z_{k',t-1}} &= - \lambda \cdot \textrm{sgn}(e_{kt}) \cdot \alpha(\beta-1)(\beta-2)|e_{kt}|^{\beta - 3}c_{kk'}, \\ \nonumber
\frac{\partial^2 \eta_{kt}}{\partial z_{kt} \partial z_{k',t-1}} &= -\lambda  \beta(\beta-1)|e_{kt}|^{\beta - 2}c_{kk'}.
\end{flalign}
We know that $|e_{lt}|^{\beta - 2}$ and $|e_{lt}|^{\beta - 3}$ are linearly independent (because their ratio, $|e_{lt}|$, is not constant). Furthermore, $|e_{lt}|^{\beta - 2}$ and $|e_{lt}|^{\beta - 3}$, with $l=1,2,...,n$, are $2n$ linearly independent functions (because of the different arguments involved).

Suppose there exist $\alpha_{l1}$ and $\alpha_{l2}$, with $l=1,2,...,n$, such that 
\begin{equation} \label{Eq:linear_dep}
\sum_{l=1}^n \big( \alpha_{l1} \mathbf{v}_{lt} + \alpha_{l2} \mathring{\mathbf{v}}_{lt} \big) = 0.
\end{equation}
It is assumed that for each $k=1, 2,..., n$, there exists at least one $k'$ such that $c_{kk'} \neq 0$. Eq.~\ref{Eq:linear_dep} then implies that for any $k$ we have
\begin{equation}
     \alpha_{k1}c_{kk'} |e_{kt}|^{\beta - 2} + \alpha_{k2} c_{kk'}|e_{kt}|^{\beta - 3} + \sum_{l \neq k} \big( \alpha_{l1}c_{lk'} |e_{lt}|^{\beta - 2} + \alpha_{l2} c_{lk'}|e_{lt}|^{\beta - 3} \big) = 0.
\end{equation}
Since $|e_{lt}|^{\beta - 2}$ and $|e_{lt}|^{\beta - 3}$, with $l=1,2,...,n$, are linearly independent and $c_{kk'} \neq 0$, to make the above equation hold, one has to set $\alpha_{k1} = \alpha_{k2} = 0$. As this applies to any $k$, we know that for Eq.~\ref{Eq:linear_dep} to be satisfied, $\alpha_{l1}$ and $\alpha_{l2}$ must be 0, for all $l=1,2,...,n$. That is,  $\mathbf{v}_{1t}, \mathring{\mathbf{v}}_{1t}, \mathbf{v}_{2t}, \mathring{\mathbf{v}}_{2t}, ..., \mathbf{v}_{nt}, \mathring{\mathbf{v}}_{nt}$ are linearly independent. The linear independence condition in Theorem \ref{Theo1} is satisfied.  Therefore 
$\mathbf{z}_{t}$ must be an invertible, component-wise transformation of a permuted version of $\hat{\mathbf{z}}_t$.
\end{proof}



\subsection{Proof for Theorem 2 and 3}

Let $\mathbf{v}_{kt}(u_r)$ be $\mathbf{v}_{kt}$, which is defined in Eq.~\ref{Eq:v}, in the $u_r$ context. Similarly, Let $\mathring{\mathbf{v}}_{kt}(u_r)$ be $\mathring{\mathbf{v}}_{kt}$ in the $u_r$ context. Let 
$$\mathbf{s}_{kt} \triangleq \Big( \mathbf{v}_{kt}(u_1)^\intercal, ..., 
\mathbf{v}_{kt}(u_m)^\intercal, 
\frac{\partial^2 \eta_{kt}({u}_{2})}{\partial z_{kt}^2 } - 
 \frac{\partial^2 \eta_{kt}({u}_{1})}{\partial z_{kt}^2 }, ...,
 \frac{\partial^2 \eta_{kt}({u}_{m})}{\partial z_{kt}^2 } - 
 \frac{\partial^2 \eta_{kt}({u}_{m-1})}{\partial z_{kt}^2 }
 \Big)^\intercal,$$
 $$\mathring{\mathbf{s}}_{kt} \triangleq \Big( \mathring{\mathbf{v}}_{kt}(u_1)^\intercal, ..., 
\mathring{\mathbf{v}}_{kt}(u_m)^\intercal, 
\frac{\partial \eta_{kt}({u}_{2})}{\partial z_{kt} } - 
 \frac{\partial \eta_{kt}({u}_{1})}{\partial z_{kt} }, ...,
 \frac{\partial \eta_{kt}({u}_{m})}{\partial z_{kt} } - 
 \frac{\partial \eta_{kt}({u}_{m-1})}{\partial z_{kt} }
 \Big)^\intercal.$$
As provided below, in our case, the identifiablity of $\mathbf{z}_t$ is guaranteed by the linear independence of the whole function vectors $\mathbf{s}_{kt}$ and $\mathring{\mathbf{s}}_{kt}$, with $k=1,2,...,n$.
However, the identifiability result in Yao et al. (2021) relies on the linear independence of only the last $m-1$ components of $\mathbf{s}_{kt}$ and $\mathring{\mathbf{s}}_{kt}$ with $k=1,2,...,n$; this linear independence is generally a much stronger condition.  

\begin{theorem}[Identifiability under Changing Causal Dynamics]
Suppose the observed processes $\mathbf{x}_t$ was generated by Eq.~\ref{Eq:generation} and that the conditional distribution $p(z_{kt} \,|\, \mathbf{z}_{t-1})$ may change across $m$ values of the context variable $\mathbf{u}$, denoted by $u_1$, $u_2$, ..., $u_m$. Suppose the components of $\mathbf{z}_t$ are mutually independent conditional on $\mathbf{z}_{-1}$ in each context. Assume that the components of $\hat{\mathbf{z}}_t$ produced by Eq.~\ref{Eq:invert} are also mutually independent conditional on $\hat{\mathbf{z}}_{t-1}$. 
If the $2n$ function vectors $\mathbf{s}_{kt}$ and $\mathring{\mathbf{s}}_{kt}$, with $k=1,2,...,n$, are linearly independent, then $\hat{\mathbf{z}}_t$ is a permuted invertible component-wise transformation of $\mathbf{z}_t$. 
\end{theorem}

\begin{proof}
As in the proof of Theorem \ref{Theo1}, because the components of $\hat{\mathbf{z}}_t$ are mutually independent conditional on $\hat{\mathbf{z}}_{t-1}$, we know that 
for $i\neq j$, 
\begin{equation}
 \label{Eq:cross2}
  \frac{\partial^2 \log p(\hat{\mathbf{z}}_t \,|\, \hat{\mathbf{z}}_{t-1}; \mathbf{u})}{\partial \hat{z}_{it} \partial \hat{z}_{jt}}
  = \sum_{k=1}^n \Big( \frac{\partial^2 \eta_{kt}(\mathbf{u})}{\partial z_{kt}^2 } \cdot \mathbf{H}_{kit}\mathbf{H}_{kjt} + \frac{\partial \eta_{kt}(\mathbf{u})}{\partial z_{kt}} \cdot \frac{\partial \mathbf{H}_{kit}}{\partial \hat{z}_{jt}} \Big)- \frac{\partial^2 \log |\mathbf{H}_t|}{\partial \hat{z}_{it} \partial \hat{z}_{jt}} \equiv 0.
\end{equation}
Compared to Eq.~\ref{Eq:cross}, here we allow $p(\hat{\mathbf{z}}_t \,|\, \hat{\mathbf{z}}_{t-1})$ to depend on $\mathbf{u}$.  Since the above equation is always 0, taking its partial derivative w.r.t. $z_{l,t-1}$ gives
 \begin{flalign} \label{Eq:linear_part1}
  \frac{\partial^3 \log p(\hat{\mathbf{z}}_t \,|\, \hat{\mathbf{z}}_{t-1}; \mathbf{u})}{\partial \hat{z}_{it} \partial \hat{z}_{jt} \partial z_{l,t-1}}
  &= \sum_{k=1}^n \Big( \frac{\partial^3 \eta_{kt}(\mathbf{u})}{\partial z_{kt}^2 \partial z_{l,t-1}} \cdot \mathbf{H}_{kit}\mathbf{H}_{kjt} + \frac{ \partial^2 \eta_{kt}( \mathbf{u})}{\partial z_{kt} \partial z_{l,t-1}}  \cdot \frac{\partial \mathbf{H}_{kit}}{\partial \hat{z}_{jt} } \Big) \equiv 0.
 \end{flalign}
Similarly, Using different values for $\mathbf{u}$ in Eq.~\ref{Eq:cross2} take the difference of this equation across them gives
\begin{flalign}
 \nonumber 
  &\frac{\partial^2 \log p(\hat{\mathbf{z}}_t \,|\, \hat{\mathbf{z}}_{t-1}; {u}_{r+1})}{\partial \hat{z}_{it} \partial \hat{z}_{jt}} - \frac{\partial^2 \log p(\hat{\mathbf{z}}_t \,|\, \hat{\mathbf{z}}_{t-1}; {u}_{r+1})}{\partial \hat{z}_{it} \partial \hat{z}_{jt}}
  \\ \label{Eq:cross3} =& \sum_{k=1}^n \Big[ \Big(\frac{\partial^2 \eta_{kt}({u}_{r+1})}{\partial z_{kt}^2 } - 
 \frac{\partial^2 \eta_{kt}({u}_{r})}{\partial z_{kt}^2 } \Big) \cdot \mathbf{H}_{kit}\mathbf{H}_{kjt} + \Big(\frac{\partial \eta_{kt}(u_{r+1})}{\partial z_{kt}} - \frac{\partial \eta_{kt}(u_{r})}{\partial z_{kt}}\Big) \cdot \frac{\partial \mathbf{H}_{kit}}{\partial \hat{z}_{jt}} \Big] \equiv 0.
\end{flalign}

Therefore, if $\mathbf{s}_{kt}$ and $\mathring{\mathbf{s}}_{kt}$, for $k=1,2,...,n$, are linearly independent, $\mathbf{H}_{kit}\mathbf{H}_{kjt}$ has to be zero for all $k$ and $i\neq j$. Then as shown in the proof of Theorem \ref{Theo1}, $\hat{\mathbf{z}}_t$ must be a permuted component-wise invertible transformation of $\mathbf{z}_t$. 
\end{proof}

\begin{theorem}[Identifiability under Observation Changes]
Suppose $\mathbf{x}_t = \mathbf{g}(\mathbf{z}_t)$ and that the conditional distribution $p(z_{k,t} \,|\, \mathbf{u})$ may change across $m$ values of the context variable $\mathbf{u}$, denoted by $u_1$, $u_2$, ..., $u_m$. Suppose the components of $\mathbf{z}_t$ are mutually independent conditional on $\mathbf{u}$ in each context. Assume that the components of $\hat{\mathbf{z}}_t$ produced by Eq.~\ref{eq:invert} are also mutually independent conditional on $\hat{\mathbf{z}}_{t-1}$. 
If the $2n$ function vectors $\mathbf{s}_{k,t}$ and $\mathring{\mathbf{s}}_{k,t}$, with $k=1,2,...,n$, are linearly independent, then $\hat{\mathbf{z}}_t$ is a permuted invertible component-wise transformation of $\mathbf{z}_t$.

\end{theorem}

\begin{proof}
As in the proof of Theorem S2, because $\mathbf{z}_t$ is not dependent on the history $\mathbf{z}_{t-1}$ so are the components of $\hat{\mathbf{z}}_t$, the conditioning on $\hat{\mathbf{z}}_t$ in Eq.~\ref{Eq:cross2} and the following equations can be removed because of the independence. This directly leads to the same conclusion as in Theorem S2. 

\end{proof}

\begin{corollary}[Identifiability under Modular Distribution Shifts] Assume the data generating process in Eq.~\ref{eq:model}. If the three partitioned latent components $\mathbf{z}_t = (\mathbf{z}_t^{\text{fix}}, \mathbf{z}_t^{\text{chg}}, \mathbf{z}_t^{\text{obs}})$ respectively satisfy the conditions in \textbf{Theorem} 1, \textbf{Theorem} 2, and \textbf{Theorem} 3, then $\mathbf{z}_{t}$ must be an invertible, component-wise transformation of a permuted version of $\hat{\mathbf{z}}_t$. 
 
\end{corollary}

\begin{proof}
Because the three partitioned subspaces $(\mathbf{z}_t^{\text{fix}}, \mathbf{z}_t^{\text{chg}}, \mathbf{z}_t^{\text{obs}})$ are conditional independent given the history and domain index, it is straightforward to factorize the joint conditional log density into three components. By using the proof in Theorem 1, 2, and 3, we can directly derive the same quantity as in Eq. 15 or Eq. 24. Therefore, if $\mathbf{s}_{kt}$ and $\mathring{\mathbf{s}}_{kt}$, for $k=1,2,...,n$, are linearly independent, $\mathbf{H}_{kit}\mathbf{H}_{kjt}$ has to be zero for all $k$ and $i\neq j$. Then as shown in the proof of Theorem \ref{Theo1}, $\hat{\mathbf{z}}_t$ must be a permuted component-wise invertible transformation of $\mathbf{z}_t$. 
\end{proof}
\subsection{Comparisons with Existing Nonlinear ICA Theories}
\label{ap:compar}

We compare our established theory with (1) \underline{PCL}~\citep{hyvarinen2017nonlinear}, (2) \underline{SlowVAE}~\citep{klindt2020towards}, (3) \underline{i-VAE}~\citep{khemakhem2020variational},  (4) \underline{GCL}~\citep{hyvarinen2019nonlinear} and (5) \underline{LEAP}~\citep{yao2021learning} in terms of their mathematical formulation and assumptions. 

\paragraph{PCL \cite{hyvarinen2017nonlinear}} The formulation of the underlying processes in PCL is in \Eqref{eq:comp0}:
\begin{equation}
\small
\label{eq:comp0}
\log p(z_{i,t}|z_{i,t-1})=G(z_{i,t}-\rho z_{i,t-1}) \quad {\rm or} \quad
\log p(z_{i,t}|z_{i,t-1})=-\lambda \left(z_{i,t} - r(z_{i,t-1})\right)^2+{\rm const},
\end{equation}
where $G$ is a non-quadratic function corresponding to the log-pdf of innovations, $\rho<1$ is the regression coefficient, $r$ is some nonlinear, strictly monotonic regression, and $\lambda$ is a positive precision parameter. 

PCL is applicable to \textbf{stationary environments} only when the sources $z_{it}$ are mutually independent (see Assumption 1 of Theorem 1 in PCL) and follow functional and distribution assumptions in \Eqref{eq:comp0}. Our formulation allows latent variables to have arbitrary, nonparametric time-delayed causal relations in between without functional form or distribution assumptions.

\paragraph{SlowVAE \cite{klindt2020towards}} The formulation of the underlying sources in SlowVAE is in \Eqref{eq:comp1}:
\begin{equation}
\label{eq:comp1}
 p(\mathbf{z}_t|\mathbf{z}_{t-1})=\prod \limits_{i=1}^d\frac{\alpha \lambda}{2 \Gamma (1/\alpha)}\exp{-(\lambda|z_{i,t}-z_{i,t-1}|^{\alpha})} \quad with \quad \alpha < 2,
\end{equation}
where the latent processes have independent, identity transitions with generalized Laplacian noises.

SlowVAE established identifiability under stationary, mutually independent processes (similar to PCL \citep{hyvarinen2017nonlinear}), in which time-delayed causal influences are not allowed. Furthermore, it assumes that the transition function of each independent process is an identity function and the process noise has generalized Laplacian distribution. Our \textbf{Theorem 1} includes both SlowVAE and PCL as special cases, in the sense that (1) we remove the functional and distributional assumptions to allow the latent processes to have nonparametric causal influences in between, and (2) our \textbf{Corollay 2}, which is an illustrative example of \textbf{Theorem 1}, further completes to \Eqref{eq:comp1} by allowing \textbf{linear time-delayed transitions} in the latent process with \textbf{non-Gaussian noises}. 

\paragraph{i-VAE \cite{khemakhem2020variational}}
 Similar to TCL~\citep{hyvarinen2016unsupervised} and GIN~\citep{sorrenson2020disentanglement}, i-VAE exploits the nonstationarity brought by class labels on the distribution of latent variables. As one can see from \Eqref{eq:comp2}, the latent variables are conditionally independent, without causal relations in between while all of our theorems consider (time-delayed) causal relations between latent variables. In addition, iVAE assumes the modulation of class labels on latent distributions is limited within the exponential family distribution. On the contrary, our nonparametric conditions (Theorems 1,2,3) allow any kind of modulation caused by fixed, changing transition dynamics or observation changes without those strong assumptions on the distribution of latent variables or noise distribution (i.e., SlowVAE \cite{klindt2020towards}). 
\begin{equation}
\label{eq:comp2}
p_{T,\lambda}(\mathbf{z}|\mathbf{u})=\prod \limits_i\frac{Q_i(z_i)}{z_i(\mathbf{u})}\exp{[\sum^k_{j=1}T_{i,j}(z_i)\lambda_{i,j}(\mathbf{u})]}
\end{equation}

\paragraph{GCL \citep{hyvarinen2019nonlinear}} The formulation of the underlying sources in GCL is in Eq~\ref{eq:gcl}, which is also described by Eqs.~4,15 in the original paper \citep{hyvarinen2019nonlinear}:

\begin{equation}
\label{eq:gcl}
p(\mathbf{z}_{t}|\mathbf{z}_{t-1})= \prod_{i=1}^d p_i(z_{i,t} | z_{i,t-1}),
\end{equation}
where the latent processes are free from the exponential family distribution assumptions but still constrained within mutually-independent processes. On the contrary, our work considers \textbf{causally-related} latent space in which cross causal relations between latent variables can be recovered. Additionally, we want to mention that we provide an antenna tube in the schema of the proof in \textbf{Theorems} 1-2-3, which is a more direct way of using sufficient variability conditions than \cite{hyvarinen2019nonlinear}.

\paragraph{LEAP \cite{yao2021learning}} LEAP in Eq.~\ref{eq:np-gen}  considers one special case of nonstationarity caused by changes in noise distributions, while our work can allow changing causal relations over context. Furthermore, because LEAP assumes all latent processes are changed across contexts, it doesn't use or benefit from the fixed time-delayed causal relations for identifiability. On the contrary, our work exploits the modular distribution changes from the fixed causal dynamics, changing dynamics, and observation changes, and hence our identifiability conditions are generally weaker than \citep{yao2021learning}.

\begin{equation}\label{eq:np-gen}
   \underbrace{ \mathbf{x}_t = g(\mathbf{z}_t) }_{\text{Nonlinear mixing}}, \quad \underbrace{z_{it} = f_i\left(\{z_{j, t-\tau} \vert z_{j, t-\tau} \in \mathbf{Pa}(z_{it}) \}, \epsilon_{it}  \right)}_{\text{Nonparametric transition}} \; with \underbrace{\epsilon_{it} \sim p_{\epsilon_i \vert \mathbf{u}}}_{\text{Nonstationary noise}}.
\end{equation}

\subsection{Discussion of the Assumptions}\label{ap:assumption}

We first explain and justify each critical assumption in the proposed conditions. We then discuss how restrictive or mild the conditions are in real applications.

\subsubsection{Linear Independence Condition}

Our proposed linear independence condition is a combination of stationary identifiability conditions (\Eqref{eq:v}) in each context $u_r$, plus the identifiability conditions for nonrecurrent influences (\Eqref{eq:vns}). The condition is essential to make each row of the Jacobian matrix $\mathbf{H}_t$ of the indeterminacy function of the learned latent space in \Eqref{eq:lind-ap} to have only one non-zero entry, thus making the learned latent variables identifiable up to permutation and component-wise invertible transformations.

In \textbf{stationary environments}, this condition essentially states that, during the generation of $\mathbf{z}_t$, if either of the two conditions is satisfied, then the linear independence condition holds in general. 

\textbf{(1)} If the history information $\mathbf{z}_{\text{Hx}} = \{\mathbf{z}_{t-\tau}\}_{\tau=1}^{L}$ up to maximum time lag $L$, and the process noise $\epsilon_t$ are coupled in a nontrivial way for generating $\mathbf{z}_t$ (e.g., heterogeneous noise process in \Eqref{eq:heteo-ap}), such that $\mathbf{z}_{\text{Hx}}$ can modulate the variance or higher-order statistics of the conditional distribution $p(\mathbf{z}_t | \mathbf{z}_{\text{Hx}})$, then the linear independence condition generally holds; 

\begin{equation} \label{eq:heteo-ap}
    z_{k,t} = q_k(\mathbf{z}_{t-1}) + \frac{1}{b_k(\mathbf{z}_{t-1})}\epsilon_{k,t}.
\end{equation}

\textbf{(2)} If the latent transition is an additive noise model (then (1) is violated) but the process noise is non-Gaussian, it will be extremely hard for the linear independence condition to be violated. Roughly speaking, for a randomly chosen conditional density function $p(z_{k,t}\,|\,\mathbf{z}_{t-1})$ in which $z_{k,t}$ is not independent from $\mathbf{z}_{t-1}$ (i.e., there is temporal dependence in the latent processes) and which does not follow an additive noise model with Gaussian noise, the chance for its specific second- and third-order partial derivatives to be linearly dependent is slim.

In \textbf{nonstationary environments}, this condition was introduced in GCL~\citep{hyvarinen2019nonlinear}, namely, ``sufficient variability'', to extend the modulated exponential families \citep{hyvarinen2016unsupervised} to general modulated distributions. Essentially, the condition says that the nonstationary regimes $\mathbf{u}$ must have a sufficiently complex and diverse effect on the transition distributions. In other words, if the underlying distributions are composed of relatively many domains of data, the condition generally holds true. For instance, in the linear Auto-Regressive (AR) model with Gaussian innovations where only the noise variance changes, the condition reduces to the statement in \citep{matsuoka1995neural} that the variance of each noise term fluctuates somewhat independently of each other in different nonstationary regimes. Then the condition is easily attained if the variance vector of noise terms in any regime is not a linear combination of variance vectors of noise terms in other regimes.

We further illustrate the condition using the example of modulated conditional exponential families in \citep{hyvarinen2019nonlinear}. Let the log-pdf $q(\z_t \vert \{\z_{t-\tau}\}, \mathbf{u})$ be a conditional exponential family distribution of order $k$ given nonstationary regime $\mathbf{u}$ and history $\mathbf{z}_{\text{Hx}}=\{\z_{t-\tau}\}$:
\begin{equation}
q(z_{it} \vert \mathbf{z}_{\text{Hx}}, \mathbf{u}) = q_i(z_{it}) + \sum_{j=1}^k q_{ij}(z_{it}) \lambda_{ij}(\mathbf{z}_{\text{Hx}}, \mathbf{u}) - \log Z(\mathbf{z}_{\text{Hx}}, \mathbf{u}),
\end{equation}

where $q_i$ is the base measure, $q_{ij}$ is the function of the sufficient statistic, $\lambda_{ij}$ is the natural parameter, and $\log Z$ is the log-partition. Loosely speaking, the sufficient variability holds if the modulation of by $\mathbf{u}$ on the conditional  distribution $q(z_{it} \vert \mathbf{z}_{\text{Hx}}, \mathbf{u})$ is not too simple in the following sense:
\begin{enumerate}

\item  Higher order of $k$ ($k>1$) is required. If $k=1$, the sufficient variability cannot hold;

\item The modulation impacts $\lambda_{ij}$ by $\mathbf{u}$ must be linearly independent across regimes $\mathbf{u}$. The sufficient statistics functions $q_{ij}$ cannot be all linear, i.e., we require higher-order statistics.

\end{enumerate}

Further details of this example can be found in Appendix B of \citep{hyvarinen2019nonlinear}. In summary, we need the modulation by $\mathbf{u}$ to have diverse (i.e., distinct influences) and complex impacts on the underlying data generation process.

\paragraph{Applicability}

By combining the stationary and nonstationary conditions, our proposed  identifiability condition is generally mild, in the sense that if there is at least one regime $r$ out of the $m$ contexts which satisfies the stationary identifiability conditions, OR, if the overall nonstationary influences are diverse and complex, thus satisfying the nonstationary identifiability conditions, the latent temporal causal processes are identifiable. For stationary conditions, the only situation where we find the latent processes unidentifiable is when the latent temporal transition is described by a Gaussian additive noise model, which violates both (1) and (2). However, for real-world data, it is very unlikely for the process noise to be perfectly Gaussian. Nonstationarity seems to be prominent in many kinds of temporal data. For example, nonstationary variances are seen in EEG/MEG, and natural video, and are closely related to changes in volatility in financial time series \citep{hyvarinen2016unsupervised}. The data that most likely satisfy the nonstationary condition is a collection of multiple trials/segments of data with different temporal dynamics in between.

\subsubsection{Independent Noise Condition}

The IN condition was introduced in the Structural Equation Model (SEM), which represents effect $Y$ as a function of direct causes $X$ and noise $E$: 
\begin{equation}
    Y = f(X, E) \quad with \quad \underbrace{X \ind E}_{\text{IN condition}}.
\end{equation}

If $X$ and $Y$ do not have a common cause, as seen from the causal sufficiency assumption of structural equation models in Chapter 1.4.1 of Pearl's book \citep{pearl2000models}, the IN condition states that the unexplained noise variable $E$ is statistically independent of cause $X$. IN is a direct result of assuming causal sufficiency in SEM. The main idea for the proof is that if IN is violated, then by the common cause principle~\citep{reichenbach1956direction}, there exist hidden confounders that cause their dependence, thus violating the causal sufficiency assumption. Furthermore, the noise terms in different variables are mutually independent for a causally sufficient system with acyclic causal relations. The main idea is that when the noise terms are dependent, it is customary to encode such dependencies by augmenting the graph with hidden confounder variables \citep{pearl2000models}, which means that the system is not causally sufficient.

This paper assumes that the underlying latent processes form a casually-sufficient system without latent causal confounders. Then, the process noise terms $\epsilon_{it}$ are mutually independent, and moreover, the process noise terms $\epsilon_{it}$ are independent of direct cause/parent nodes $\mathbf{Pa}(z_{it})$ because of time information (the causal graph is acyclic because of the temporal precedence constraint).

\paragraph{Applicability}
Loosely speaking, if there are no latent causal confounders in the (latent) causal processes and the sampling frequency is high enough to observe the underlying dynamics, then the IN condition assumed in this paper is satisfied in a causally-sufficient system and, moreover, there is no instantaneous causal influence (because of the high enough resolution). At the same time, we acknowledge that there exist situations where the resolution is low and there appears to be instantaneous dependence. However, several pieces of work deal with causal discovery from measured time series in such situations; see. e.g., \cite{granger1987implications, Subsampling_ICML15, Danks14, Aggre_UAI17}.
In case there are instantaneous causal relations among latent causal processes, one would need additional sparsity or minimality conditions to recover the latent processes and their relations, as demonstrated in \cite{Silva06,Adams_21}. How to address the issue of instantaneous dependency or instantaneous causal relations in the latent processes will be one line of our future work.

\subsubsection{Causal Influences between Observed Variables}

Although causal discovery between observed variables is not the main focus of our work, our model can discover causal relations between observed variables as a special case, thanks to the nonlinear mixing function assumed in this paper. In our formulation in \Eqref{eq:model}, we assume the observations $x_t$ are nonlinear, invertible mixtures of latent processes $z_t$. However, if in the data generating process, one observed variable $x_i$ has direct causal edges with the other observed variable $x_j$, the mixing function that generates $x_i$ and $x_j$ will just be reduced to identity mappings of $z_i$ and $z_j$, which is a special case of the nonlinear invertible mixing function in \Eqref{eq:model}. 

\subsection{Extension to Multiple Time Lags}\label{sec:mlag-ap}

 For the sake of simplicity, we consider one lag for the latent processes in Section 3 (and only in Section 3). Our identifiability proof can actually be applied for arbitrary lags directly. For instance, in the stationary case in \Eqref{eq:v}, one can simply re-define $\eta_{kt} \triangleq \log p(z_{k, t} | \mathbf{z}_{Hx})$, where  $\mathbf{z}_{Hx}$ denotes the lagged latent variables up to maximum time lag $L$. We plug it into \Eqref{eq:v}, and take derivatives with regard to $z_{1, t-\tau}, …, z_{n, t-\tau}$, which can be any latent temporal variables at lag $\tau$, instead of $z_{1, t-1}, …, z_{n, t-1}$. If there exists one $\tau$ (out of the $L$ lags) that satisfies the condition, then the stationary latent processes are identifiable. Similarly, for \Eqref{eq:vns}, one can simply re-define $\eta\_{kt}(u_r)  \triangleq \log p(z_{k, t} | \mathbf{z}_{Hx}, u_r)$ and plug it into \Eqref{eq:vns}. No extra changes are needed.
 
\color{black}

\section{Experiment Settings}\label{ap:experiment}

\subsection{Datasets}
\subsubsection{Synthetic Dataset Generation}\label{ap:synthetic}

We consider three representative simulation settings to validate the identifiability results under fixed causal dynamics, changing causal dynamics, and modular distribution shift which contains fixed dynamics, changing dynamics, and global changes together in the latent processes. For synthetic datasets with fixed and changing causal dynamics, we set latent size $n = 8$. For the modular shift dataset, we add one dimension for global observation changes. The lag number of the process is set to $L = 2$. The mixing function $g$ is a random three-layer MLP with LeakyReLU units.

\paragraph{Fixed Causal Dynamics}

For the fixed causal dynamics. We generate 100,000 data points according to Eq. \eqref{eq:heteo}, where the latent size is $n=8$, lag number of the process is $L = 2$.
We apply a 2-layer MLP with LeakyReLU as the state transition function. The process noise are sampled from i.i.d. Gaussian distribution ($\sigma=0.1$). The process noise terms are coupled with the history information through multiplication with the average value of all the time-lagged latent variables. 

\paragraph{Changing Causal Dynamics}

We use a Gaussian additive noise model with changes in the influencing strength as the latent processes. To add changes, we vary the values of the first layer of the MLP across the 20 segments and generate 7,500 samples for each segment. The entries of the kernel matrix of the first layer are uniformly distributed between $[-1,1]$ in each domain. 

\paragraph{Modular Distribution Shifts}

The latent space of this dataset is partitioned into 6 fixed dynamics components under the heterogeneous noise model, 2 changing components with changing causal dynamics and 1 component modulated by domain index only. The fixed and changing dynamics components follow the same generating procedures above. The global change component is sampled from i.i.d Gaussian distribution whose mean and variance are modulated by domain index. In particular, distribution mean terms are uniformly sampled between $[-1,1]$ and variance terms are uniformly sampled between $[0.01, 1]$. 

\subsubsection{Real-world Dataset}\label{ap:real}
\paragraph{Modified Cartpole} The Cartpole problem \citep{huang2021adarl} ``consists of a cart and a vertical pendulum attached to the cart using a passive pivot joint. The cart can move left or right. The task is to prevent the vertical pendulum from falling by putting a force on the cart to move it left or right. The action space consists of two actions: moving left or right.'' 

\begin{figure}[h]
 \centering
 \includegraphics[width=0.75\textwidth]{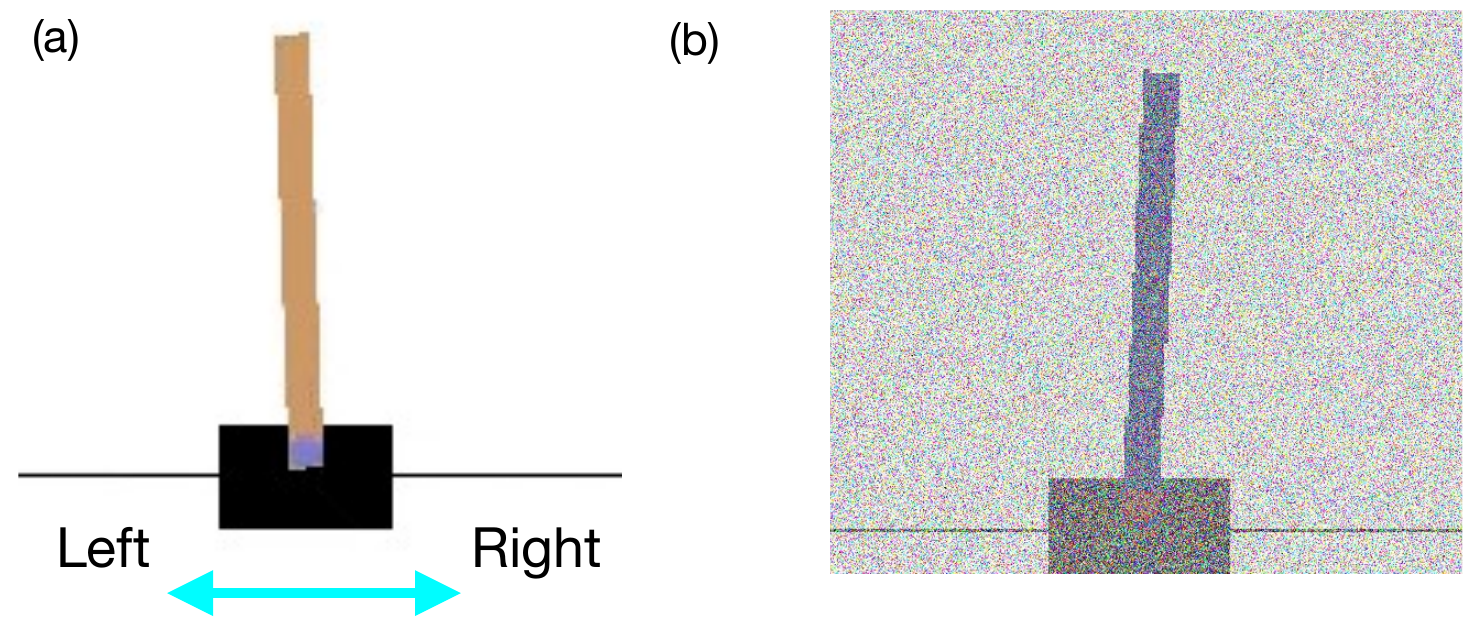}
 \caption{``Visual examples of Cartpole game and change factors. (a) Cartpole game; (b) Modified Cartpole game with Gaussian noise on the image. The light blue arrows are added to show the direction in which the agent can move.'' Figure source: \citep{huang2021adarl}.
 \label{Figure: visual_cartpole}}
\end{figure} 
 
The original dataset \citep{huang2021adarl} introduces ``two change factors respectively for the state transition dynamics $\theta^{\text{dyn}}_k$: varying gravity and varying mass of the cart, and a change factor in the observation function $\theta^{\text{obs}}_k$ that is the image noise level.
Fig.~\ref{Figure: visual_cartpole} gives a visual example of Cartpole game, and the image with Gaussian noise. The images of the varying gravity and mass look exactly like the original image.
Specifically, in the gravity case, we consider source domains with gravity $g= \{5, 10, 20, 30, 40\}$. We take into account both interpolation (where the gravity in the target domain is in the support of that in source domains) with $g = \{15\}$,  and extrapolation (where it is out of the support w.r.t. the source domains) with $g = \{55\}$. 
Similarly, we consider source domains where the mass of the cart is $m = \{0.5, 1.5, 2.5, 3.5, 4.5\}$, while in target domains it is $m=\{1.0, 5.5\}$. In terms of changes on the observation function, we add Gaussian noise on the images with variance $\sigma = \{0.25, 0.75, 1.25, 1.75, 2.25 \}$ in source domains, and $\sigma = \{0.5, 2.75\}$ in target domains. The detailed settings in both source and target domains are in Table~\ref{Table: cartpole_setting}.''

 \begin{table}[h]
    \centering
    \begin{tabular}{c|c|c|c}
    \toprule
        & Gravity & Mass & Noise \\ \hline
    Source domains   &  $\{5, 10, 20, 30, 40\}$ & $\{0.5, 1.5, 2.5, 3.5, 4.5\}$ & $\{0.25, 0.75, 1.25, 1.75, 2.25\}$ \\ \hline
    Interpolation set  &  $\{15\}$ & $\{1.0\}$ & $\{0.5\}$ \\ \hline
    Extrapolation set  &  $\{55\}$ & $\{5.5\}$ & $\{2.75\}$ \\ \bottomrule
    \end{tabular}
    \caption{``The settings of source and target domains for modified Cartpole experiments'' \citep{huang2021adarl}.}
    \label{Table: cartpole_setting}
\end{table}



\paragraph{CMU-Mocap}
 CMU MoCap (\url{http://mocap.cs.cmu.edu/}) is an open-source human motion capture dataset with various motion capture recordings (e.g., walk, jump, basketball, etc.) performed by over 140 subjects. In this work, we fit our model on 11 trials of ``walk'' recordings (Subject \#8). Skeleton-based measurements have 62 observed variables corresponding to the locations of joints (e.g., head, foot, shoulder, wrist, throat, etc.) of the human body at each time step.
 
\subsection{Mean Correlation Coefficient}
MCC is a standard metric for evaluating the recovery of latent factors in ICA literature. MCC first calculates the  absolute values of the correlation coefficient between every ground-truth factor against every estimated latent variable. Pearson correlation coefficients or Spearman's rank correlation coefficients can be used depending on whether componentwise invertible nonlinearities exist in the recovered factors. The possible permutation is adjusted by solving a linear sum assignment problem in polynomial time on the computed correlation matrix. 
\section{Implementation Details}\label{ap:implememtation}

\subsection{Modular Prior Likelihood Derivation}\label{ap:derive}

Let us start with an illustrative example of stationary latent causal processes consisting of two time-delayed latent variables, i.e., $\mathbf{z}_t = [z_{1,t}, z_{2,t}]$ with maximum time lag $L=1$, i.e., $z_{i,t} = f_i(\mathbf{z}_{t-1}, \epsilon_{i,t})$ with mutually independent noises. Let us write this latent process as a transformation map $\mathbf{f}$ (note that we overload the notation $f$ for transition functions and for the transformation map):

\begin{equation}
    \begin{bmatrix}
    z_{1,t-1} \\
    z_{2,t-1} \\
    z_{1,t} \\
    z_{2,t} \\
    \end{bmatrix} 
    =\mathbf{f} \left(
    \begin{bmatrix}
    z_{1,t-1} \\
    z_{2,t-1} \\
    \epsilon_{1,t} \\
    \epsilon_{2,t}
    \end{bmatrix}
    \right).
\end{equation}

By applying the change of variables formula to the map $\mathbf{f}$, we can evaluate the joint distribution of the latent variables $p(z_{1,t-1}, z_{2,t-1}, z_{1,t},z_{2,t})$ as:
\begin{equation}\label{eq:cvt-1}
p(z_{1,t-1}, z_{2,t-1}, z_{1,t},z_{2,t}) = p(z_{1,t-1}, z_{2,t-1}, \epsilon_{1,t},\epsilon_{2,t}) / \left|\det \mathbf{J}_\mathbf{f}\right|,
\end{equation}
where $\mathbf{J}_\mathbf{f}$ is the Jacobian matrix of the map $\mathbf{f}$, which is naturally a low-triangular matrix:
$$
\mathbf{J}_\mathbf{f} = 
\begin{bmatrix}
1 & 0 & 0 & 0\\
0 & 1 & 0 & 0\\
\frac{\partial z_{1,t}}{\partial z_{1,t-1}} & \frac{\partial z_{1,t}}{\partial z_{2,t-1}} & \frac{\partial z_{1,t}}{\partial \epsilon_{1,t}} & 0 \\ 
\frac{\partial z_{2,t}}{\partial z_{1,t-1}} & \frac{\partial z_{2,t}}{\partial z_{2,t-1}} & 0 & \frac{\partial z_{2,t}}{\partial \epsilon_{2,t}}
\end{bmatrix}.
$$
Given that this Jacobian is triangular, we can efficiently compute its determinant as $\prod_i \frac{\partial z_{i,t}}{\partial \epsilon_{i,t}}$. Furthermore, because the noise terms are mutually independent, and hence $\epsilon_{i,t} \perp \epsilon_{j,t}$ for $j \neq i$ and $\epsilon_t \perp \mathbf{z}_{t-1}$, we can write the RHS of Eq.~\ref{eq:cvt-1} as:

\begin{equation}\label{eq:example}
\begin{aligned}
    p(z_{1,t-1}, z_{2,t-1}, z_{1,t},z_{2,t}) &= p(z_{1,t-1}, z_{2,t-1}) \times p(\epsilon_{1,t},\epsilon_{2,t}) / \left|\det \mathbf{J}_\mathbf{f}\right| \quad (\text{because }\epsilon_t \perp \mathbf{z}_{t-1})\\
    &= p(z_{1,t-1}, z_{2,t-1}) \times \prod_i p(\epsilon_{i,t}) / \left|\det \mathbf{J}_\mathbf{f}\right| \quad (\text{because }\epsilon_{1,t} \perp \epsilon_{2,t})
\end{aligned}    
\end{equation}

Finally, by canceling out the marginals of the lagged latent variables $p(z_{1,t-1}, z_{2,t-1})$ on both sides, we can evaluate the transition prior likelihood as:

\begin{equation}\label{eq:example-likelihood}
p( z_{1,t},z_{2,t} \vert z_{1,t-1}, z_{2,t-1}) = \prod_i p(\epsilon_{i,t}) / \left|\det \mathbf{J}_\mathbf{f}\right| = \prod_i p(\epsilon_{i,t}) \times \left|\det \mathbf{J}_\mathbf{f}^{-1}\right|.
\end{equation}

Now we generalize this example and derive the modular prior likelihood below. 

\paragraph{Fixed Causal Dynamics}  Let $\{f^{-1}_{s}\}_{s=1,2,3...}$ be a set of learned inverse fixed dynamics transition functions that take the estimated latent causal variables in the fixed dynamics subspace and lagged latent variables, and output the noise terms, i.e., $\hat{\epsilon}_{s,t} = f^{-1}_{s}\left(\hat{z}_{s,t}^{\text{fix}}, \{\hat{\mathbf{z}}_{t-\tau}\} \right)$. 

Design transformation $\mathbf{A} \rightarrow \mathbf{B}$ with  low-triangular Jacobian as follows:
\vspace{-2mm}
\begin{align}
\small
\underbrace{
\begin{bmatrix}
\hat{\mathbf{z}}_{t-L},
\hdots,
\hat{\mathbf{z}}_{t-1},
\hat{\mathbf{z}}_{t}^{\text{fix}}
\end{bmatrix}^{\top}
}_{\mathbf{A}}
\textrm{~mapped to~}
\underbrace{
\begin{bmatrix}
\hat{\mathbf{z}}_{t-L},
\hdots,
\hat{\mathbf{z}}_{t-1},
\hat{\mathbf{\epsilon}}_{s,t}
\end{bmatrix}^{\top}
}_{\mathbf{B}},
~
with~
\mathbf{J}_{\mathbf{A} \rightarrow \mathbf{B}} = 
\begin{pmatrix}
\mathbb{I}_{nL} & 0 \\
* & \text{diag}\left(\frac{\partial f^{-1}_{s,i}}{\partial \hat{z}_{it}^{\text{fix}}}\right)
\end{pmatrix}.
\end{align}
Similar to Eq.~\ref{eq:example-likelihood}, we can obtain the joint distribution of the estimated fixed dynamics subspace as:
\vspace{-0.5mm}
\begin{align}
&\log p(\mathbf{A}) = \underbrace{\log p\left(\hat{\mathbf{z}}_{t-L}, \hdots, \hat{\mathbf{z}}_{t-1}\right) + \sum_{i=1}^n \log p(\hat{\epsilon}_{s,t})}_\text{Because of mutually independent noise assumption} + \log \left(\lvert \det \left(\mathbf{J}_{\mathbf{A} \rightarrow \mathbf{B}}\right) \rvert \right) \label{eq:np-joint}.\\
&\log p\left(\hat{\mathbf{z}}_t^{\text{fix}} \vert \{\hat{\mathbf{z}}_{t-\tau}\}_{\tau=1}^L\right) = \sum_{i=1}^n \log p(\hat{\epsilon}_{s,t})+ \sum_{i=1}^n \log \Big| \frac{\partial f^{-1}_{s}}{\partial \hat{z}_{s,t}}\Big|  
\end{align}

\paragraph{Changing Causal Dynamics} The differences from fixed dynamics are that the learned inverse changing dynamics transition functions take additional learned change factors of the context as input arguments to out the noise terms, i.e, $\hat{\epsilon}_{c,t} = f^{-1}_{c}\left(\hat{z}_{c,t}^{\text{chg}}, \{\hat{\mathbf{z}}_{t-\tau}\}, \mathbf{u}_k \right) =  f^{-1}_{c}\left(\hat{z}_{c,t}^{\text{chg}}, \{\hat{\mathbf{z}}_{t-\tau}\}, \boldsymbol{\theta}_k^{\text{dyn}} \right)$. 
\begin{align}
\log p\left(\hat{\mathbf{z}}_t^{\text{chg}} \vert \{\hat{\mathbf{z}}_{t-\tau}\}_{\tau=1}^L, \mathbf{u}_k\right) = \sum_{i=1}^n \log p(\hat{\epsilon}_{c,t} \vert \mathbf{u}_k)+ \sum_{i=1}^n \log \Big| \frac{\partial f^{-1}_{c}}{\partial \hat{z}_{c,t}}\Big|
\end{align}

\paragraph{Observation Changes} The global observation changes are captured by the learned inverse $f_o^{-1}$,which takes the estimated latent subspace and the learned change factors for global observation $\boldsymbol{\theta}_k^{\text{obs}}$ of context $k$, and output random noise, i.e, $\hat{\epsilon}_{o,t} = f^{-1}_{o}\left(\hat{z}_{o,t}^{\text{obs}}, \mathbf{u}_k \right) =  f^{-1}_{c}\left(\hat{z}_{c,t}^{\text{chg}}, \boldsymbol{\theta}_k^{\text{obs}} \right)$.
\begin{align}
\log p\left(\hat{\mathbf{z}}_t^{\text{obs}} \vert \mathbf{u}_k\right) = \sum_{i=1}^n \log p(\hat{\epsilon}_{o,t} \vert \mathbf{u}_k)+ \sum_{i=1}^n \log \Big| \frac{\partial f^{-1}_{o}}{\partial \hat{z}_{o,t}}\Big|.
\end{align}
\subsection{Comparisons with AdaRL \cite{huang2021adarl}}

In terms of the implementation, in our work, we enforce the independent noise or conditional independence condition explicitly in \Eqref{eq:np-trans} (derived in Appendix \ref{ap:derive}) for the identifiability of the latent processes. But disentanglement is not the main goal of AdaRL so it used Mixture Density Network (MDN) to approximate the transition prior. Our framework is simpler than AdaRL in the inference module (we use only $\mathbf{x}_t$ to infer $\mathbf{z}_t$) and the loss function (we don’t have the prediction branch or the sparsity loss), thanks to the nonlinear ICA formulation in \Eqref{eq:model}. Our identifiability conditions do not rely on the sparsity constraints in the underlying data generating process.

\subsection{Network Architecture}\label{ap:arch}
We summarize our network architecture below and describe it in detail in Table~\ref{tab:arch-details} and Table~\ref{tab:arch-cnndetails}.

\begin{table}[ht]
\caption{ Architecture details. BS: batch size, T: length of time series, i\_dim: input dimension, z\_dim: latent dimension, LeakyReLU: Leaky Rectified Linear Unit.}
\label{tab:arch-details}
\resizebox{\textwidth}{!}{%
\begin{tabular*}{1.05\textwidth}{@{\extracolsep{\fill}}|lll|}

\toprule
\textbf{Configuration} & \textbf{Description} &  \textbf{Output} \\
\toprule
\toprule
\textbf{1. MLP-Encoder} &  Encoder for Synthetic Data & \\
\toprule
Input: $\x_{1:T}$ & Observed time series & BS $\times$ T $\times$ i\_dim \\
Dense & 128 neurons, LeakyReLU & BS $\times$ T $\times$ 128\\
Dense & 128 neurons, LeakyReLU & BS $\times$ T $\times$ 128 \\
Dense & 128 neurons, LeakyReLU & BS $\times$ T $\times$ 128 \\
Dense & Temporal embeddings & BS $\times$ T $\times$ z\_dim \\
\toprule
\toprule
\textbf{2. MLP-Decoder} & Decoder for Synthetic Data & \\
\toprule
Input: $\hat{\z}_{1:T}$ & Sampled latent variables & BS $\times$ T $\times$ z\_dim \\
Dense & 128 neurons, LeakyReLU & BS $\times$ T $\times$ 128 \\
Dense & 128 neurons, LeakyReLU & BS $\times$ T $\times$ 128 \\
Dense & i\_dim neurons, reconstructed $\mathbf{\hat{x}}_{1:T}$ & BS $\times$ T $\times$ i\_dim \\
\toprule
\toprule
\textbf{5. Factorized Inference Network} & Bidirectional Inference Network & \\
\toprule
Input & Sequential embeddings & BS $\times$ T $\times$ z\_dim \\
Bottleneck & Compute mean and variance of posterior & $\mathbf{\mu}_{1:T}, \mathbf{\sigma}_{1:T}$ \\
Reparameterization & Sequential sampling & $\hat{\z}_{1:T}$ \\
\toprule
\toprule
\textbf{6. Modular Prior} & Nonlinear Transition Prior Network & \\
\toprule
Input & Sampled latent variable sequence $\hat{\z}_{1:T}$ & BS $\times$ T $\times$ z\_dim \\
InverseTransition & Compute estimated residuals $\hat{\epsilon}_{it}$ & BS $\times$ T $\times$ z\_dim \\
JacobianCompute & Compute $\log \left(\lvert \det \left(\mathbf{J}\right) \rvert \right)$ & BS\\
\bottomrule
\end{tabular*}
}
\end{table}

\begin{table}[ht]
\caption{ Architecture details on CNN encoder and decoder. BS: batch size, T: length of time series, h\_dim: hidden dimension, z\_dim: latent dimension, F: number of filters, (Leaky)ReLU: (Leaky) Rectified Linear Unit.}
\label{tab:arch-cnndetails}
\resizebox{\textwidth}{!}{%
\begin{tabular*}{1.05\textwidth}{@{\extracolsep{\fill}}|lll|}
\toprule
\textbf{Configuration} & \textbf{Description} &  \textbf{Output} \\
\toprule
\toprule
\textbf{3.1.1 CNN-Encoder} & Feature Extractor & \\
\toprule
Input: $\x_{1:T}$ & RGB video frames & BS $\times$ T $\times$ 3 $\times$ 64 $\times$ 64 \\
Conv2D & F: 32, BatchNorm2D, LeakyReLU & BS $\times$ T $\times$ 32 $\times$ 64 $\times$ 64 \\
Conv2D & F: 32, BatchNorm2D, LeakyReLU & BS $\times$ T $\times$ 32 $\times$ 32 $\times$ 32 \\
Conv2D & F: 32, BatchNorm2D, LeakyReLU & BS $\times$ T $\times$ 32 $\times$ 16 $\times$ 16 \\
Conv2D & F: 64, BatchNorm2D, LeakyReLU & BS $\times$ T $\times$ 64 $\times$ 8 $\times$ 8 \\
Conv2D & F: 64, BatchNorm2D, LeakyReLU & BS $\times$ T $\times$ 64 $\times$ 4 $\times$ 4 \\
Conv2D & F: 128, BatchNorm2D, LeakyReLU & BS $\times$ T $\times$ 128 $\times$ 1 $\times$ 1 \\
Dense & F: 2 * z\_dim = dimension of hidden embedding & BS $\times$ T $\times$ 2 * z\_dim \\
\toprule
\toprule
\textbf{4.1 CNN-Decoder} & Video Reconstruction & \\
\toprule
Input: $\z_{1:T}$ & Sampled latent variable sequence & BS $\times$ T $\times$ z\_dim \\
Dense & F: 128 , LeakyReLU & BS $\times$ T $\times$ 128 $\times$ 1 $\times$ 1\\
ConvTranspose2D & F: 64, BatchNorm2D, LeakyReLU & BS $\times$ T $\times$ 64 $\times$ 4 $\times$ 4 \\
ConvTranspose2D & F: 64, BatchNorm2D, LeakyReLU & BS $\times$ T $\times$ 64 $\times$ 8 $\times$ 8 \\
ConvTranspose2D & F: 32, BatchNorm2D, LeakyReLU & BS $\times$ T $\times$ 32 $\times$ 16 $\times$ 16 \\
ConvTranspose2D & F: 32, BatchNorm2D, LeakyReLU & BS $\times$ T $\times$ 32 $\times$ 32 $\times$ 32 \\
ConvTranspose2D & F: 32, BatchNorm2D, LeakyReLU & BS $\times$ T $\times$ 32 $\times$ 64 $\times$ 64 \\
ConvTranspose2D & F: 3, estimated scene $\mathbf{\hat{x}}_{1:T}$ & BS $\times$ T $\times$ 3 $\times$ 64 $\times$ 64 \\
\bottomrule
\end{tabular*}
}
\end{table}

\subsection{Hyperparameter and Training}

\paragraph{Hyperparameter Selection}

The hyperparameters include $\beta$, which is the weight of KLD terms, as well as the latent size $n$ and maximum time lag $L$. We use the ELBO loss to select the best pair of $\beta$ because low ELBO loss always leads to high MCC. We always set a larger latent size than the true latent size. This is critical in real-world datasets because restricting the latent size will hurt the reconstruction performances and over-parameterization makes the framework robust to assumption violations. For the maximum time lag $L$, we set it by the rule of thumb. For instance, we use $L=2$ for temporal datasets with a latent physics process (e.g, cartpole, cmu-mocap).

\paragraph{Training Details} 
 The models were implemented in \texttt{PyTorch} 1.8.1. The VAE network is trained using AdamW optimizer for a maximum of 50 epochs and early stops if the validation ELBO loss does not decrease for five epochs. A learning rate of 0.002 and a mini-batch size of 64 are used.  We have used three random seeds in each experiment and reported the mean performance with standard deviation averaged across random seeds.


\paragraph{Computing Hardware} 

We used a machine with the following CPU specifications: Intel(R) Core(TM) i7-7700K CPU @ 4.20GHz; 8 CPUs, four physical cores per CPU, a total of 32 logical CPU units. The machine has two GeForce GTX 1080 Ti  GPUs with 11GB GPU memory.
 
\paragraph{Reproducibility}

We've included the code for the framework and all experiments in the supplement. We plan to release our code under the MIT License after the paper review period. 



\section{Additional Experiment Results}

\begin{figure}[ht]
    \centering
    \includegraphics[width=0.95\linewidth]{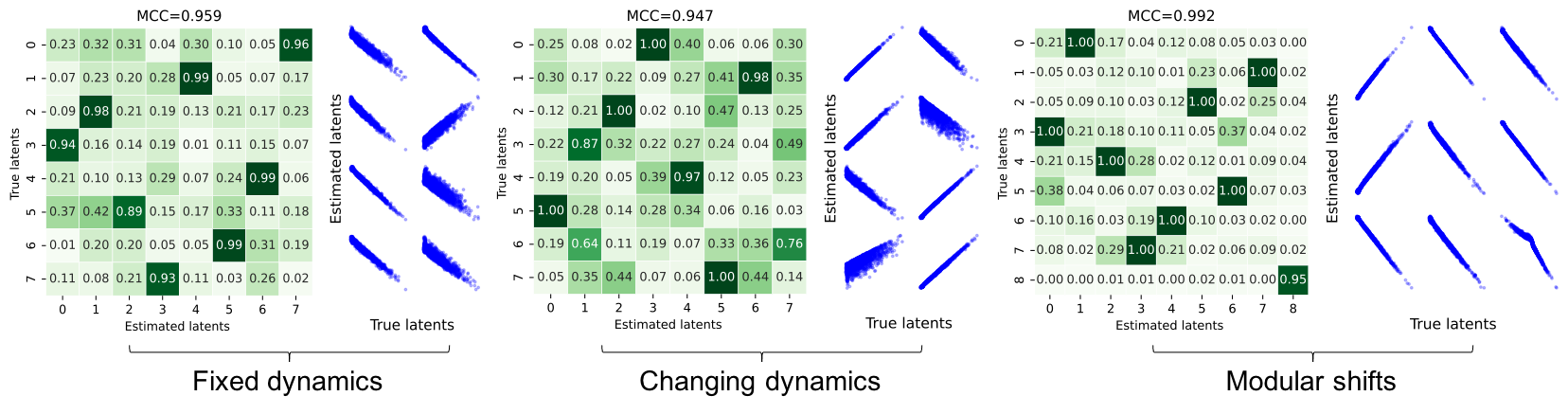}
    \caption{Results for three synthetic datasets: in each block, the left shows the MCC for causally-related and the left are scatterplots between estimated and true factors.  \label{fig:simulated}}
    \vspace{-0.5cm}
\end{figure}

\end{document}